\documentclass{article} 
\usepackage{iclr2027_conference,times}
\usepackage{booktabs}
\usepackage{multirow}

\usepackage{amsmath,amsfonts,bm}

\def\eqref#1{equation~\ref{#1}}

\def\1{\bm{1}}

\DeclareMathAlphabet{\mathsfit}{\encodingdefault}{\sfdefault}{m}{sl}
\SetMathAlphabet{\mathsfit}{bold}{\encodingdefault}{\sfdefault}{bx}{n}

\newcommand{\R}{\mathbb{R}}

\usepackage{float}
\usepackage{placeins}
\usepackage[hidelinks]{hyperref}
\usepackage{url}
\usepackage{xurl}
\usepackage{graphicx}
\usepackage{refcount}
\usepackage{amsmath}
\usepackage{amsthm}
\usepackage{enumitem}
\usepackage{xcolor}
\usepackage{tikz}
\usepackage{fontawesome5}

\definecolor{buttonbg}{HTML}{F7F7F7}
\definecolor{buttonfg}{HTML}{8C8C8C}

\newcommand{\paperbutton}[3]{%
  \href{#1}{%
    \tikz[baseline=(button.base)]{
      \node[
        fill=buttonbg,
        text=buttonfg,
        rounded corners=4pt,
        inner xsep=7pt,
        inner ysep=4pt,
        font=\small\ttfamily,
        text height=1.6ex,
        text depth=0.3ex
      ] (button) {#2\hspace{0.45em}#3};
    }%
  }%
}

\newtheorem{lemma}{Lemma}
\newtheorem{theorem}[lemma]{Theorem}

\title{World Modeling through \\ Spectral Alignment}
\newcounter{proofparent}
\newtheorem{prooflemma}{Lemma}[proofparent]

\newcommand{\lemmasfor}[1]{%
    \setcounterref{proofparent}{#1}%
    \setcounter{prooflemma}{0}%
}
\newenvironment{restatedresult}[3]{%
  \begin{trivlist}
  \item[\hskip\labelsep\bfseries #1~\ref{#2} (#3).]\itshape
}{%
  \end{trivlist}
}

\author{
Holger Molin\thanks{Equal contribution.} \\
Stanford University \\
\And
William Peng\footnotemark[1] \\
Stanford University \\
\And
Marco Bagatella \\
ETH Zürich \\
\And
Randall Balestriero \\
Brown University
}

\newcommand{\norm}[1]{\left\lVert#1\right\rVert}
\newenvironment{appendixexperiment}{%
  \par
  \noindent\begin{minipage}[t]{\linewidth}%
}{%
  \end{minipage}\par
}
\iclrfinalcopy 
\begin{document}

\maketitle
\pagestyle{plain}
\thispagestyle{plain}
\par\vspace{-1.2em}

{%
  \setlength{\parskip}{0pt}
  \centering
  \paperbutton
    {https://spectral-alignment.github.io/}
    {\faGlobe}
    {Website}
  \hspace{0.5em}
  \paperbutton
    {https://github.com/7peng/SpecWM}
    {\faGithub}
    {Code}
  \par
}

\begin{abstract}
How do we tell a world model what information to preserve? We study this question through the lens of spectral representation learning. We introduce SpecWM, a JEPA framework that explicitly specifies desired relationships between observations through a target similarity kernel and trains the encoder to reproduce those relationships in its embedding geometry. Different kernels provide different supervision, ranging from physical state information to a fully self-supervised temporal kernel requiring only the ordering of observations within trajectories. We characterize the global optima of the objective, establish conditions for linear recoverability and invariance to nuisance information, and connect temporal-kernel similarity to the behavioral successor measure under an idealized setting. Empirically, the choice of kernel strongly controls which physical quantities are recoverable from the representation. More importantly, temporal spectral alignment improves CEM planning success over LeWorldModel from 52\% to 68\% on OGBench Cube, 49\% to 77\% on OGBench Scene, and 29\% to 51\% on CALVIN. These results suggest that explicitly specifying representation geometry provides a simple way to steer what latent world models preserve for control.

\end{abstract}

\section{Introduction}
The goal of a world model is to encode sufficient information to predict the consequences of an agent's actions~\citep{HaoWorldModels}.
At the same time, complex real-world observations contain large amounts of task-irrelevant information~\citep{zhang2021learninginvariantrepresentationsreinforcement}.
Learning a compact representation therefore hinges on deciding which aspects of an observation to retain.
Joint Embedding Predictive Architectures (JEPAs;~\citealp{lecun2022path}) predict outcomes in a latent space, rather than reconstructing high-dimensional observations in pixel space, allowing learned representations to abstract away details of the input.
Due to the lack of explicit supervision or grounding in raw observation space, training JEPAs is generally non-trivial. The training objective must implicitly guide which distinctions the model preserves without requiring it to reproduce every detail of an observation.

Recent work has made substantial progress toward stable JEPA training through regularizing the embedding distribution \citep{LeJEPA, LeWorldModel}. However, avoiding collapse alone does not ensure that a representation preserves information needed for control, and the degree to which this happens is dictated implicitly by self-prediction~\citep{tang2022understandingselfpredictivelearningreinforcement}. For instance, an encoder may distinguish robot configurations while ignoring the objects the robot must manipulate. We seek a direct method to express these relationships and understand their consequences.

We approach this problem through the lens of spectral representation learning \citep{laplacianeigenmaps,bengio2004learning}, which historically focuses on extracting representations from pairwise relationships between samples. Prior work has shown that several contrastive and non-contrastive objectives generally admit spectral interpretations \citep{haochen, balestriero2022contrastivenoncontrastiveselfsupervisedlearning}. Building on this perspective, we make the pairwise notion of similarity between two observations the center of a new training objective for latent dynamics modeling.

We introduce a \textbf{spectral alignment} framework for training world models to preserve a chosen similarity structure.
Leveraging a notion of distances between observations (the \textit{teacher}), we define a target similarity kernel. We then train an encoder (the \textit{student}) to learn embeddings whose pairwise similarities match the target and an action-conditioned predictor to output future latents.
This framework can be instantiated through different teachers requiring varying degrees of supervision: when some information about the underlying system is known, encoders can be trained based on physical distances, but the natural ordering of states within a trajectory is also sufficient to provide a training signal. We name the resulting world model SpecWM.

\begin{figure}[t]
    \centering
    \includegraphics[width=1\linewidth]{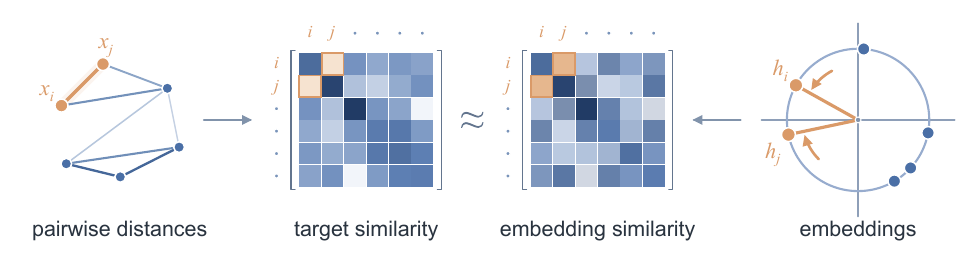}
    \caption{Spectral alignment framework. Pairwise distances (target) between samples define a kernel that an encoder is trained to reconstruct through embedding similarities.}
    \label{fig:abstract}
\end{figure}

Theoretically, we first study how the choice of teacher shapes the learned representation. When the teacher is specified for all pairs of states, we characterize the global optima as the kernel's spectral embeddings. This connection allows us to identify which quantities can be recovered from the representation via a linear probe. When teacher supervision is only available within a trajectory, we establish conditions under which matching the teacher removes observation noise from the embeddings. Finally, for the specific choice of temporal kernels, we connect the expected similarity target to the behavioral successor measure, which describes how likely the data-collection policy is to visit a given state in the future.

Practically, we instantiate this framework across visual manipulation environments spanning OGBench  Scene and Cube~\citep{park2025ogbenchbenchmarkingofflinegoalconditioned} as well as CALVIN~\citep{mees2022calvinbenchmarklanguageconditionedpolicy}.
Empirically, we confirm that kernel choice substantially impacts which physical quantities are linearly recoverable. Further, when leveraging a self-supervised temporal teacher, whose only source of supervision is the order of observations in a trajectory, the resulting world models outperform existing JEPA world models~\citep{LeWorldModel} in planning performance.

Our contributions can be summarized as follows:
\begin{enumerate}
    \item We introduce a novel world modeling objective based on spectral alignment. 
    \item We formally study this framework through the lens of spectral representation learning, with a particular focus on temporal kernels.
    \item We present a broad empirical evaluation on visual control tasks and measure the ability of different kernels to recover physical quantities and drive planning performance.
\end{enumerate}

\section{Related Work}

\textbf{World models} learn how environments respond to actions, enabling agents to plan and learn from predicted outcomes \citep{HaoWorldModels}. Generative approaches predict future observations, often by modeling dynamics in a latent space and decoding predictions into images. PlaNet uses learned dynamics to evaluate candidate action sequences, while the Dreamer series trains policies on imagined trajectories \citep{PlaNet, DreamerV4}. More broadly, generative models support interactive environments and spatial world generation \citep{Genie3, HunyuanWorld2, worldlabs2026atlas}, with world-action models such as Cosmos Policy, DreamZero, and LingBotVA extending these capabilities to action generation for robot control \citep{cosmospolicy,DreamZero,lingbotVA}. In contrast, reconstruction-free world models focus prediction on representations without requiring them to reproduce observations. DINO-WM predicts features from a pretrained visual encoder \citep{dinoWM} while JEPAs such as LeWorldModel learn the representations alongside the dynamics \citep{lecun2022path, LeWorldModel}.

\textbf{JEPA} as proposed in \cite{lecun2022path} features the central question of what anti-collapse mechanism is used to guide which information is retained and which is abstracted away. In the cases of I-JEPA \citep{I-JEPA} and V-JEPA 2.1 \citep{V-JEPA21}, the target encoder is an exponential moving average of the context encoder with a stop-grad, making collapse difficult to attain but not actually a non-optimal solution. LeJEPA \citep{LeJEPA} solves anti-collapse by enforcing the distribution of latents to be an isotropic Gaussian, which prevents trivial collapse but does not require any particular information to be retained.

\textbf{Spectral representation learning}
is intimately tied to self-supervised learning. \citet{haochen} interpret contrastive learning as spectral decomposition of an augmentation graph. \citet{balestriero2022contrastivenoncontrastiveselfsupervisedlearning} extend this perspective, showing that several objectives including InfoNCE~\citep{infonce}, VICReg~\citep{VICReg}, and Barlow Twins~\citep{zbontar2021barlowtwinsselfsupervisedlearning} recover classical spectral embedding methods where the graph is defined implicitly using the choice of augmentations. Our spectral alignment objective extends this work by making previously implicit relations between samples the central object of study. 

We cover additional families of related work in Appendix~\ref{app:extra_related}.

\section{Preliminaries}
\label{sec:preliminaries}
Let $\{x_i\}^n_{i=1}$ be a collection of inputs and let $f_\theta : \mathcal{X} \to \mathbb{R}^{d}$ be an encoder parameterized by $\theta$. We denote its representations by $h_i = f_\theta(x_i)$, which we stack as rows of $H \in \mathbb{R}^{n \times d}$. A symmetric matrix $K \in \mathbb{R}^{n \times n}$ can be used to specify the pairwise relationships between inputs. For instance, in an augmentation-based method, $K_{ij}$ may be designated as $1$ for two views produced from the same observation, and as $0$ otherwise. More generally, $K$ can be interpreted as the result of a symmetric similarity function --- i.e., a \textit{kernel} --- $k:\mathcal{X}\times\mathcal{X}\to \mathbb{R}$, which defines relationships through $K_{ij}=k(x_i,x_j)$.

\textbf{Spectral embeddings. \;}
Given a teacher kernel $K$ specifying the desired inner product between every pair of inputs in a collection, we can construct a coordinate vector for each input that realizes these pairwise relationships. Such a construction is called a \emph{spectral embedding}.
Suppose that \(K\) is symmetric and positive semidefinite, meaning that a valid spectral embedding exists. $K$ then has eigendecomposition
\begin{equation}
K = \sum _{\ell = 1}^r\lambda_\ell u_\ell u_\ell^\top =U_r\Lambda_r U_r^\top,
\label{decomp_def}
\end{equation}
where \(r=\operatorname{rank}(K)\), \(\lambda_\ell>0\) is the
\(\ell\)-th positive eigenvalue of \(K\), \(u_\ell\in\mathbb{R}^n\)
is its corresponding unit eigenvector,
\(U_r=[u_1,\ldots,u_r]\), and
\(\Lambda_r=\operatorname{diag}(\lambda_1,\ldots,\lambda_r)\). A spectral embedding of \(K\) is given by
$
H_{\mathrm{spec}}
=
U_r\Lambda_r^{1/2}
$
where the \(i\)-th row \(h_i^{\mathrm{spec}}\) provides the coordinates of sample \(i\). Exactly satisfying all the inner products of $K$ requires an embedding dimension $d\ge r.$ When \(d<r\), we get $H_d=U_d\Lambda_d^{1/2}$, for which \(H_dH_d^\top=U_d\Lambda_dU_d^\top\) is the best rank-\(d\) approximation of \(K\) in Frobenius norm. Each eigenvalue determines the relative contribution of its corresponding eigenvector to the kernel, as seen in Equation \ref{decomp_def}, making eigenvalue rescaling a natural reweighting operation. To construct a particular spectral embedding, one selects $d$ eigenvectors and assigns each a nonnegative scale $s_\ell$ yielding 
\begin{equation}
    H_{\text{spec}}
    =
    U_d\operatorname{diag}(s_1,\ldots,s_d),
    \qquad
    h_i^{\text{spec}}
    =
    \bigl(s_1u_1(i),\ldots,s_du_d(i)\bigr).
    \label{spec_def}
\end{equation}
The canonical embedding of \(K\) is recovered through \(s_\ell=\sqrt{\lambda_\ell}\); other choices reweight the spectrum.

\textbf{Linear probing.\;}
To assess which physical quantities are linearly accessible from a representation, we fit a linear map from fixed embeddings to target labels. Given targets $Y\in\mathbb{R}^{n\times p}$, such as object positions, the optimal unregularized probing error is $\min_{W\in\mathbb{R}^{d\times p}}\|Y-HW\|_F^2$.
This error is zero precisely when every column of $Y$ lies in the column space of $H$, denoted $\operatorname{col}(H)$.

\section{Method}
\label{sec:method}

\begin{figure}[t]
    \centering
    \includegraphics[width=1\linewidth]{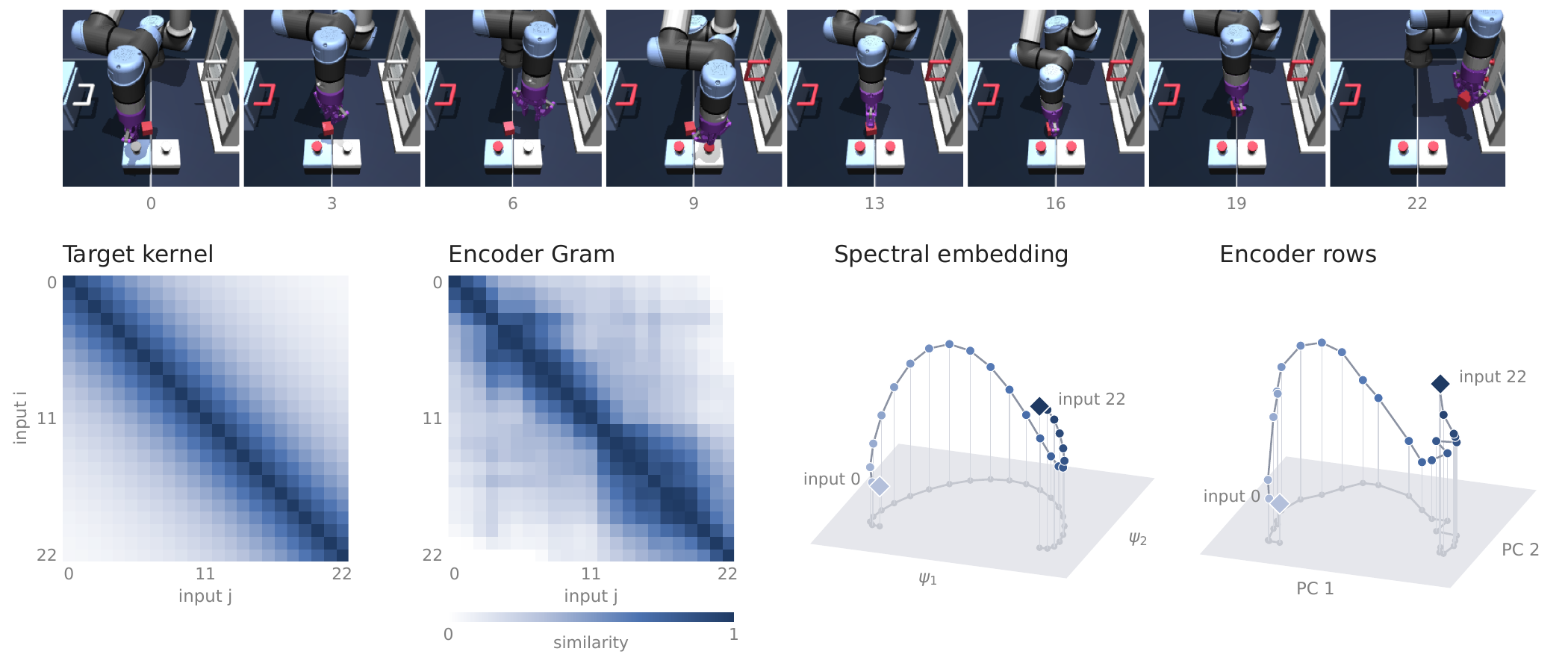}
    \caption{Spectral alignment on a single trajectory. Top: sampled observations from OGBench Scene~\citep{park2025ogbenchbenchmarkingofflinegoalconditioned}. Bottom: the temporal target kernel and encoder Gram matrix (before centering), the spectral embedding of $K$ (top three eigenvectors), and the encoder rows $h_i$ (top three principal components), colored by frame index. Matching pairwise similarities encourages the encoder to recover the target's geometry.}
    \label{fig:pipeline}
\end{figure}

We will state the general framework for arbitrary kernels, and later specialize it to practical kernels.

\textbf{Building a target.\;}
We start from a distance $d$ between pairs of observations (which could be derived from ground-truth measurements, if available, or simply by counting the number of observations that separate two observations in the data). 
Considering a sequence of observations $\{x_i\}_{i=1}^n$, we leverage the distance $d$ to build a target \emph{kernel} $K_{ij}$. In practice, we adopt a Laplacian kernel to convert distances into similarities:
\begin{equation}
K_{ij} = \exp\big(-d(x_i, x_j)/\sigma\big),
\end{equation}

with values in $(0, 1]$. Its slope at $d = 0$ is $-1/\sigma$, so that observations that are close according to $d$ remain distinguishable, while pairs more than a few $\sigma$ apart have similarity near $0$. 
The bandwidth $\sigma$ therefore sets the range of distances the target resolves; we set it to the median teacher distance over the pairs the kernel is defined on.
We then double center $K$ to form $\tilde K$:
\begin{equation}
    \Tilde{K}_{ij} \;=\; K_{ij} \;-\; \frac{1}{n}\sum_{k} K_{ik} \;-\; \frac{1}{n}\sum_{k} K_{kj} \;+\; \frac{1}{n^2}\sum_{k,l} K_{kl}.
\end{equation}
Without centering, the kernel would be dominated by a positive baseline, since kernel values concentrate above zero, and regressing this target would favor matching that shared baseline rather than the variation around it.

Still, if a few eigenvalues dominate the kernel, the encoder can achieve a small squared error by matching those directions while largely ignoring the others. We therefore \textit{temper} the kernel by reshaping its spectrum before it is used as a target. We write the eigendecomposition of the centered kernel as $\Tilde K = U \Lambda U^T$ with eigenvalues in decreasing order, and define $\Lambda^\alpha$ by raising the leading $m$ positive eigenvalues to the power $\alpha \in [0,1]$ and setting all remaining entries to zero. The target is then
\begin{equation}
    K^\alpha = n\frac{U\Lambda^\alpha U^T}{\text{tr}\Lambda^\alpha}.
\end{equation}
At $\alpha=1$, each eigenvector occupies the target's variance according to its eigenvalue fraction. At $\alpha =0$, each eigenvector's share is uniform.
In practice, we will use $\alpha = 1/2$, thus using the square root of the kernel matrix as target, and truncate the spectrum at $m=128$.

\textbf{Training the encoder. \;}
We implement the encoder $f_\theta$ as a deep neural network that maps each observation $x_i$ to an embedding $h_i \in \mathbb{R}^d$.
It will act as a \textit{student}, and regress the kernel designed from the teacher distance. 
We first rescale each embedding to the length the target prescribes
\begin{equation}
\label{eq:rescaling}
    \bar{h}_i = \sqrt{K^\alpha_{ii}} \frac{h_i}{\norm{h_i}},
\end{equation}

and then compute the student's Gram matrix --- the symmetric matrix of inner products of its rows ---  $S = \bar{H}\bar{H}^T$.
Similarly to the kernel, we double center the Gram matrix as $\Tilde S$, and regress the target through a spectral loss:
\begin{equation}
    \mathcal{L}_\text{spec} = \norm{\Tilde S- K^\alpha}^2_F.
\end{equation}

\textbf{Prediction anchor.\;}
While $\mathcal{L}_\mathrm{spec}$ is sufficient for training an encoder, planning in latent space requires learning a predictor, i.e., a neural network $g_\phi: \mathbb{R}^d \times \mathcal{A} \to \mathbb{R}^d$ parameterized by $\phi$ and trained to predict the latent outcome of applying some action from an action space $\mathcal{A}$ after receiving a given observation. For stability, we apply a parameter-free RMS normalization to the encoder output, so that every embedding has norm $\sqrt{d}$. While the spectral loss is invariant to this step, it fixes the scale of the embeddings seen by the predictor and the planner.
We train a predictor by a simple regression objective in latent space, which is also backpropagated through the encoder, such that the total loss is
\begin{equation}
    \mathcal{L} = \mathcal{L}_\text{spec}
    - \lambda \sum_i \big(\cos(G,H')\big)_i,
    \label{eq:total_loss}
\end{equation}
where $G$ stacks predictions $g_\phi(f_\theta(x_i), a_i)$, $H'$ stacks representations of next observations $f_\theta(x_i')$ for each training transition $(x_i, a_i, x_i')$, and $\cos(\cdot)$ extracts the vector of row-wise cosine similarities from the two matrices.

\textbf{Planning. \;}
Once a spectral world model is trained through the objective in Equation~\ref{eq:total_loss}, it can be directly deployed in the standard planning pipeline for JEPA world models~\citep{LeWorldModel}.
In practice, we use the Cross-Entropy Method (CEM) \citep{rubinstein1999cross} to optimize action sequences from the current observation $x_0$:
candidate action sequences are sampled from a Gaussian distribution, and the world model predicts the representation of observations $T$ steps ahead for each. Candidates are then scored according to similarity to the embedding of a goal observation $g$:
\begin{equation}
\cos\left(g_\phi^T(f_\theta(x_0), a_{0}^{T-1}), f_\theta(g)\right),
\label{eq:planning}
\end{equation}
where $g_\phi^T$ denotes recursive application of the predictor over $T$ steps, and $a_0^{T-1}$ is a sequence of actions.
The action distribution is refit on the best candidates for a fixed number of iterations. We then execute a prefix of the selected action sequence before replanning from the new observation. We detail hyperparameters in Appendix \ref{reproducibility}.

\paragraph{Kernel design.}
\label{sec:kernel_design}
Often, robotics data is labeled by informative metadata such as proprioception or depth masks. In such cases, that metadata can be used to construct a \textit{supervised} kernel. We evaluate three such kernels: state, depth, and proprioception. For each, we use the Euclidean distance between the corresponding feature vectors as the teacher distance.
The state kernel uses the simulator state $s_i$ underlying each observation $x_i$ as its feature vector. Although this information is generally unavailable in real-world settings, it lets us evaluate a teacher with access to the full physical state.
The proprioception and depth kernels use the robot's proprioceptive state and the flattened depth mask as feature vectors, respectively.

When metadata is not available, the ordering of observations in trajectories itself can form an \textit{unsupervised} kernel.
We use a temporal kernel, which relies on absolute difference between indices for observations ordered in a trajectory: $d(x_i, x_j) = |i - j|$.
We also evaluate a quasimetric kernel as a close relative to the temporal kernel.
This kernel relies on a distance $d_q$ learned through quasimetric reinforcement learning~\citep{wang2023optimalgoalreachingreinforcementlearning,iqepaper}, which estimates temporal distances between observations according to the optimal goal-reaching policy.
While this potentially introduces estimation errors, it labels state pairs by the minimum number of steps between states, instead of expressing the temporal separations observed in the data, and can thus compensate for poor data quality.
Since the student's Gram matrix is symmetric, we symmetrize the learned quasimetric $d(x_i,x_j)=d_q(x_i \to x_j) + d_q(x_j \to x_i)$.

While supervised kernels can describe relationships between all pairs of inputs, i.e., are \textit{fully specified}, the unsupervised kernels we use only express relationships for pairs of inputs in the same trajectory. In this
\textit{underspecified} setting, cross-trajectory similarities are excluded from the training objective. We also use a different normalization of the spectral loss: we omit tempering and the target-diagonal rescaling in Equation~\ref{eq:rescaling}. Each clip's squared error is normalized by the mean squared magnitude of its centered target before averaging across clips. While representations are fixed up to an orthogonal transformation in the fully specified case, underspecification allows the model to form its own abstractions~\citep{haochen2023theoreticalstudyinductivebiases}. 

\section{Theoretical Analysis} 
\label{theory}
In this section, we build upon previous analyses~\citep{bengio2004learning, balestriero2022contrastivenoncontrastiveselfsupervisedlearning,blier2021learningsuccessorstatesgoaldependent} and provide formal statements regarding the spectral alignment objective and some intuition on their practical relevance. Full proofs of each theorem are provided in Appendix \ref{proofs}. \looseness -1

\subsection{The spectral loss has a known optimum}
A foundational property of our objective is that in the fully specified setting, an embedding table that minimizes $\mathcal{L}_\text{spec}$ recovers the row-normalized spectral embeddings of $K^\alpha$~\citep{bengio2004learning}.

\begin{theorem}[Spectral recovery]
\label{thm:recovery}
Suppose $m \leq d$ and $K^\alpha_{ii}>0$ for all $i$. Then $\min_H \mathcal{L}_\text{spec}(H) = 0$, which is attained
exactly at
\begin{equation}
    \label{eq:minimizers}
    H^\star = \operatorname{diag}(c)\, H_\text{spec}\, Q,
\end{equation}
with $c \in \mathbb{R}^n_{>0}$, $Q \in \mathbb{R}^{d\times d}$ an orthogonal matrix, and $H_\text{spec}$ the spectral embedding of $K^\alpha$ with square-root eigenvalue scaling.
\end{theorem}
This result is useful as the spectral embeddings of the kernel can be calculated through eigendecomposition as described in Section~\ref{sec:preliminaries}. We can thus evaluate the probing qualities of spectral embeddings derived from arbitrary kernels without ever having to train a model. 

\subsection{Downstream probing guarantees} 

The ability to probe a given quantity from spectral embeddings depends on its relationship with the chosen kernel.
Following prior work \citep{balestriero2022contrastivenoncontrastiveselfsupervisedlearning}, we show that a kernel constructed from label relationships can make the labels themselves linearly recoverable, without directly regressing encoder outputs onto them.
Intuitively, this suggests that the choice of a certain kernel encourages embeddings to extract the information used to construct it; this motivates the supervised kernels evaluated in Section~\ref{kernel_ablations}.

\begin{theorem}[Linear recoverability]
\label{thm:linear-recovery}
Let $Y\in\mathbb R^{n\times p}$ satisfy $\mathbf1^\top Y=0$, row norms $\|y_i\|=1$, and $\operatorname{rank}(Y)\le d$. For $K=YY^\top$, $\alpha=1$, and all positive spectral directions retained, every global minimizer $H^\star$ of $\mathcal L_{\text{spec}}$ subject to $\|h_i\|=\sqrt d$ satisfies
\begin{equation}
    \label{eq:zero-linear-probe}
    \min_{W\in\mathbb R^{d\times p}}\|Y-H^\star W\|_F^2=0.
\end{equation}
\end{theorem}

\subsection{Task-irrelevant information}
\label{sec:kni}

A central goal of a JEPA is selective invariance--- dropping irrelevant parts of an observation~\citep{lecun2022path}. We define task-irrelevant information as any aspect of an observation the teacher kernel is invariant to. For example, a cost-to-go kernel for robotic manipulation may ignore the exact patterns of leaves rustling in the background. We ask whether matching such a teacher also makes the encoder ignore these details. Formally, consider a sequence of observations
\begin{equation*}
\{x_i\}^n_{i=1}=\{(s_i,n_i)\}^n_{i=1},
\end{equation*}
where $s_i$ is the task-relevant component and $n_i$ is the task-irrelevant component. Let $\nu(\cdot\mid s_i)$ denote the distribution of task-irrelevant values at state $s_i$, and let $Z$ collect the task-relevant trajectory. We establish invariance under a coverage assumption: nuisance values that are individually possible at two states must also be possible together in a trajectory visiting both states. This excludes a background color that varies across clips but is always identical between frames of a clip. 

\textbf{Coverage. (A1)}
\label{assump:coverage}
Fix a task-relevant trajectory $Z$ and two distinct observation indices $i$ and $j$. Draw $n \sim \nu(\cdot\mid s_i)$ and $n' \sim \nu(\cdot\mid s_j)$ independently. For almost every $Z$, all $i\neq j$, and every measurable set $E$ of nuisance pairs, assume
\begin{equation*}
\Pr\bigl((n,n')\in E\mid s_i,s_j\bigr)>0
\quad\Longrightarrow\quad
\Pr\bigl((n_i,n_j)\in E\mid Z\bigr)>0.
\end{equation*}
Conditioning on the trajectory may change which nuisance pairs are likely, but cannot make an otherwise possible set of pairs impossible. This assumption covers independently sampled sensor noise and allows correlated disturbances, such as stochastic foliage motion or rainfall, when all otherwise possible combinations remain possible. It excludes persistent trajectory-level attributes, such as a background identity or lighting that varies across trajectories but is fixed within each one.

\begin{theorem}[Abstraction of task-irrelevant information]
\label{thm:kni-budget}
Under the coverage assumption, fixed-norm encoder outputs, and the target-diagonal rescaling in Equation~\ref{eq:rescaling}, if each trajectory's target matrix $K^\alpha$ has strictly positive diagonal entries almost surely and
$\mathbb E\,[\mathcal L_{\mathrm{spec}}]=0$, then for almost every clip and every frame $i$,
\begin{equation}
    f_\theta((s_i,n))=f_\theta((s_i,n'))
\end{equation}
for almost every pair $(n,n')$ drawn independently from $\nu(\cdot\mid s_i)$.
\end{theorem}

At zero expected loss, varying the task-irrelevant component while holding the task-relevant component fixed leaves the representation unchanged.

\subsection{Temporal kernels and the successor measure}
\label{sec:successor}

Having analyzed what information spectral alignment keeps or discards for arbitrary kernels, we now focus on the temporal kernel, and show that the representations it retrieves can be desirable for planning with a world model. This analysis replaces the free embedding tables above with a state-conditioned encoder $h:\mathcal{S}\to\{z\in\mathbb R^d:\norm{z}=1\}$. We will show that $\mathcal{L}_{\mathrm{spec}}$ corresponds to a regression problem, whose target
can be related to the (behavioral) successor measure~\citep{blier2021learningsuccessorstatesgoaldependent}: the discounted cumulative probability of the data generation process to visit a given state in the future.
This connects the embedding similarity used in planning to how readily states lead to one another under the behavior represented in the data.

We assume that each trajectory in the dataset is generated by a stationary Markov policy $\pi_\beta$ in a Markov decision process with a finite state space $\mathcal{S}$, and that the induced Markov chain $P$ is irreducible and aperiodic, admitting a stationary distribution $\mu$, from which the initial state $x_0$ is sampled.
We can introduce a discount factor derived from the Laplacian kernel and its temperature: $\gamma:=e^{-1/\sigma} \in (0,1)$, and then define the behavioral successor measure as
\begin{equation}
    M^\beta(x,x') \;=\; \sum_{t=0}^{\infty}\gamma^{t}\,
    \Pr\!\left(x_t=x'\mid x_0=x,\pi_\beta\right).
\end{equation}
The training data is generated by drawing trajectories $(x_0,\dots,x_{T-1})$ from the induced Markov chain, and indices $(i,j)$ uniformly from $\{0,\dots,T-1\}^2$.
For simplicity, we finally assume that the student and the teacher are not centered and that the latter does not undergo tempering: $\alpha=1$, and no truncation is applied. Let $s(x,x')=h(x)^\top h(x')$ denote embedding similarity. Under these assumptions, we consider the population-level objective:
\begin{equation}
  \mathcal L_{\mathrm{spec}}^{\mathrm{pop}}(h) \;=\; \mathbb E_{\tau, i,j} \Big[
  \big(s(x_i,x_j) - \gamma^{|i-j|}\big)^2\Big],
\end{equation}
with $\tau:=(x_0, \dots, x_{T-1}) \sim \mathcal{D}$ and $(i,j)\sim\mathcal U(\{0,\dots,T-1\}^2)$.
We also define the expected teacher similarity of the state pair as
\begin{equation}
    R_T(x, x') = \mathbb E_{\tau, i, j} [\gamma^{|i-j|}|x_i=x, x_j=x']. 
\end{equation}
In Appendix \ref{proofs}, we demonstrate that $\mathcal{L}_{\mathrm{spec}}^{\mathrm{pop}}(h)$ is equivalent to minimizing $\mathbb{E}_{\tau, i, j} \big[(s(x_i, x_j)- R_T(x_i, x_j))^2\big]$, with the embedding similarity $s(x_i, x_j) = h(x_i)^\top h(x_j)$.
We can then study the target $R_T$ in the limit of long trajectories, and connect it to the successor measure $M^\beta$ of the behavioral policy $\pi_\beta$.

\begin{theorem}[Population target]
\label{prop:pop-target}
Under the setup above, with a finite state space, an irreducible, aperiodic Markov chain $P$ induced by $\pi_\beta$ with stationary distribution $\mu$ and $x_0\sim\mu$, for all $x,x'\in\mathcal S$, as $T\to\infty$,
\begin{equation}
    T\,R_T(x,x') \;\longrightarrow\;
    \frac{M^\beta(x,x')}{\mu(x')} + \frac{M^\beta(x',x)}{\mu(x)}
    - \frac{\mathbf 1[x=x']}{\mu(x)} .
    \label{eq:pop-target-def}
\end{equation}
Moreover, if the induced Markov chain is reversible, i.e., $\mu(x)P(x'|x)=\mu(x')P(x|x')$,
\begin{equation}
    T\,R_T(x,x') \;\longrightarrow\;
    \frac{2M^\beta(x,x')}{\mu(x')} - \frac{\mathbf 1[x=x']}{\mu(x')} .
    \label{eq:pop-target-def-reversible}
\end{equation}
\end{theorem}

Embedding similarity is thus encouraged to regress a target that tracks a symmetrized likelihood ratio, holding the behavioral successor measure as its numerator.
This relates the planning score in Equation~\ref{eq:planning} to discounted visitation between predicted states and the goal, in both temporal directions and relative to their stationary frequencies.

\section{Experiments}

Our empirical evaluation revolves around three visual manipulation environments: OGBench Cube and Scene \citep{park2025ogbenchbenchmarkingofflinegoalconditioned} and CALVIN \citep{mees2022calvinbenchmarklanguageconditionedpolicy}. Scene and CALVIN extend the evaluation beyond the simpler manipulation settings considered in LeWorldModel~\citep{LeWorldModel}, introducing multiple objects and articulated fixtures whose states must be represented for control. These environments directly test whether the learned representation preserves information about the surrounding scene as well as the robot itself. We primarily compare SpecWM to LeWM~\citep{LeWorldModel}, and we instantiate SpecWM for each of the kernels introduced in Section~\ref{sec:kernel_design}. A generally observed phenomenon in world modeling is that predictors tend to generalize poorly to random actions when trained on only expert data \citep{fu2021d4rldatasetsdeepdatadriven, kidambi2021morelmodelbasedoffline, yu2020mopomodelbasedofflinepolicy}. Motivated by this, we generate failure data for each environment and append it to the canonical datasets; this decision is ablated in Appendix \ref{app:data_ablations}. 

Our main points of comparison are success rate with CEM planning (Section \ref{planning}) and linear recoverability of physical quantities from embeddings (Section \ref{probing_results}). Further, we provide an ablation of various kernels in both fully specified and underspecified settings in Section \ref{kernel_ablations}. We extend these results with additional ablations in Appendix~\ref{app:extra_exp}.
For CALVIN, we augment training clips with temporally distant observations and use a multiscale kernel to capture block movements that short clips often miss. Appendix~\ref{reproducibility} details this construction, failure data generation, and evaluation settings.

\subsection{Planning \label{planning}}

With temporal and quasimetric kernels, SpecWM improves upon LeWM by an average of 21 percentage points in planning success rate (Figure \ref{fig:cem}). Intuitively, the planner ranks predicted outcomes by their embedding similarity to the goal, which spectral alignment shapes to reflect temporal or quasimetric proximity. Overall, we observe that a learned quasimetric kernel is generally on par with a readily available temporal kernel, but either may be preferable depending on the environment. 

\begin{figure}[h]
  \centering
  \includegraphics[width=1\linewidth]{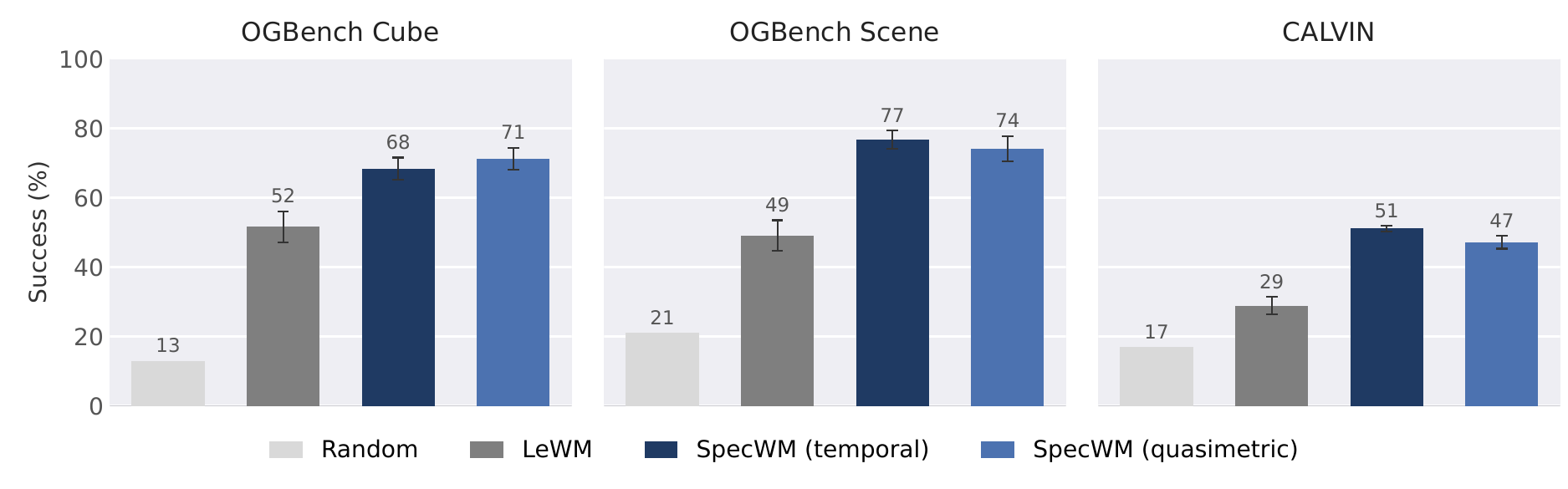}
  \caption{Average CEM planning success rate on 100 fixed tasks per environment. Bars report the mean over the 20k, 24k and 28k checkpoints and 3 seeds, with whiskers showing one cross-seed standard deviation. SpecWM consistently outperforms LeWM.}
  \label{fig:cem}
\end{figure}

\subsection{Probing}
\label{probing_results}
To isolate representation quality from prediction accuracy, we now evaluate linear probing performance for the position of objects across environments.
We consider a held-out set of observations that the encoder was never trained on for each dataset, and define $5$ train-evaluation splits.
For each split, we fit a ridge regressor for each state variable, choose regularization by cross-validation, and report linear recoverability on the evaluation set.
We report $R^2$ between predicted and ground-truth values, averaged over variables and splits.

As reported in Figure \ref{fig:probes}, we find that SpecWM does not always beat LeWM on probing, and probe quality does not directly correlate with planning performance. 

\begin{figure}[h]
  \centering
  \includegraphics[width=1\linewidth]{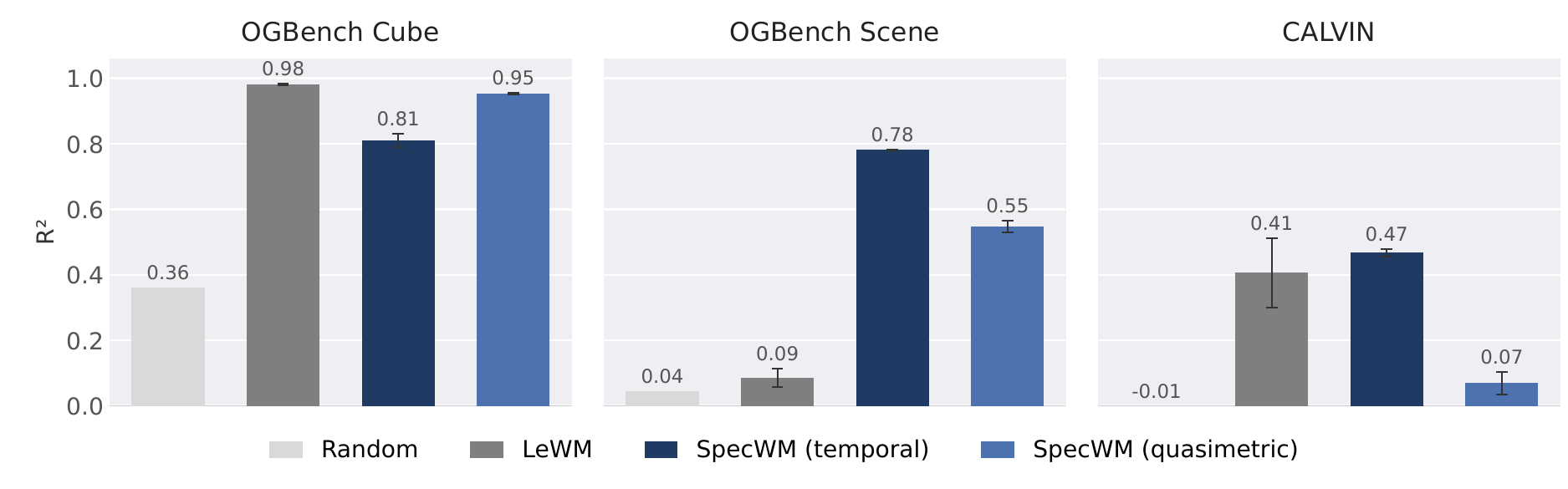}
  \caption{Linear probe $R^2$ of the manipulated object's position from each model's final checkpoint embedding. Bars report the mean over 3 seeds, with whiskers showing one cross-seed standard deviation. Random is an untrained encoder of the same architecture.}
  \label{fig:probes}
\end{figure}
\subsection{Kernel Ablations \label{kernel_ablations}}

While the previous evaluations focus on the temporal and quasimetric targets, we now extend the investigation to the full set of kernels we introduced in Section~\ref{sec:method}, including supervised, fully specified variants (i.e., state, depth, and proprioception), and unsupervised, underspecified variants (i.e., temporal, quasimetric).
The latter will help in quantifying the information lost due to lack of ground-truth supervision.
Table \ref{tab:kernel_probes} reports linear probes on all of the fixtures in the Scene environment.

The fully specified state kernel achieves the highest average $R^2$ (0.93). Under underspecified supervision, the temporal kernel nearly matches the state kernel and outperforms the quasimetric kernel. Proprioceptive kernels preserve the robot's configuration but largely omit object states, supporting Theorem \ref{thm:kni-budget}. For the fully specified kernels, we note that kernel eigenvectors are indeed strongly recoverable from the learned embeddings (Appendix~\ref{app:spectral_recovery}).

\begin{table}[h!]
\centering
\caption{Linear probe $R^2$ on held-out OGBench
Scene episodes. Bold marks the best result in each row and
results within 0.01 of it.}
\label{tab:kernel_probes}
\setlength{\tabcolsep}{4pt}
\renewcommand{\arraystretch}{1.25}
\begin{tabular}{lcccccccc}
& \multicolumn{5}{c}{Underspecified}
& \multicolumn{3}{c}{Fully specified}\\
\cmidrule(lr){2-6} \cmidrule(lr){7-9}
& Temporal & Quasimetric & State & Proprio & Depth
& State & Proprio & Depth\\
\midrule
Scene (avg)
& 0.84 & 0.78 & 0.85 & 0.16 & 0.15
& \textbf{0.93} & 0.16 & 0.50 \\
\quad \texttt{cube}
& 0.79 & 0.51 & 0.57 & 0.10 & 0.09
& \textbf{0.96} & 0.08 & 0.09  \\
\quad \texttt{arm}
& 0.67 & 0.66 & 0.69 & 0.69 & 0.60
& \textbf{0.73} & 0.67 & 0.59  \\
\quad \texttt{drawer}
& 0.91 & 0.90 & \textbf{0.98} & 0.01 & 0.01
& \textbf{0.99} & 0.01 & 0.87  \\
\quad \texttt{window}
& 0.87 & 0.88 & \textbf{0.99} & 0.00 & 0.01
& \textbf{0.99} & 0.00 & 0.91 \\
\quad \texttt{buttons}
& 0.96 & 0.94 & \textbf{1.00} & 0.02 & 0.01
& \textbf{0.99} & 0.02 & 0.04 \\
\bottomrule
\end{tabular}
\end{table}

\section{Conclusion}
We present world modeling through spectral alignment as a framework for designing the spectral methods that have so far been implicit in self-supervised learning.
The core idea is to directly specify the pairwise relationships that representations should preserve through a teacher kernel. We explore theoretical connections to fundamental results in spectral manifold learning and, for a specific choice of kernel, to the likelihood of reaching goals expressed in the data distribution.
Further, we validate the method empirically and analyze its abstraction properties.
We find that readily available kernels result in strong planning performance, which is not entirely correlated with probing accuracy. Pursuing a better understanding of the mechanisms enabling effective planning would be valuable. We view the study of which kernels are most conducive to large-scale self-supervised training as the next primary direction of work. We further view kernels involving language for semantic distances as an interesting future direction.

\subsection*{AI use statement}

In this work, we used generative AI tools to implement research methods and develop and revise code. We also used these tools to discuss research methodology and theoretical results, identify relevant literature, assist in proving theoretical statements, and draft and edit portions of the manuscript. We have reviewed all AI-assisted work. We take responsibility for the final content of this work, including text, claims, or artifacts produced with the aid of generative AI.

\subsection*{Reproducibility statement}

To ensure effective reproducibility, all code necessary for this paper is available at (\url{https://anonymous.4open.science/r/SpecWM/}). Full proofs of presented theorems can be found in Appendix \ref{proofs}. The dataset generation process along with evaluation protocols and training settings are described in Appendix \ref{reproducibility} as well as available in the code.

\bibliography{iclr2027_conference}
\bibliographystyle{iclr2027_conference}

\newpage
\appendix
\flushbottom

\section{Extended Related Work}
\label{app:extra_related}

\textbf{Kernel alignment}, the normalized Frobenius inner product between a candidate kernel and a target similarity matrix, was introduced by \citet{NIPS2001_1f71e393} as a criterion for kernel selection. \citet{KAMPFFMEYER2018816} later introduced alignment of an autoencoder's embedding Gram matrix to a provided kernel. We extend this paradigm by removing the reconstruction term and use action conditioned prediction as an additional shaping force for the latent. Their unsupervised kernel derives similarity from multiscale clustering of input features, which need not reflect the dynamical relationships relevant to control. We instead study both supervised and unsupervised kernels with a more direct relevance for control and world modeling. Furthermore, our trajectory-local formulation leaves cross-trajectory relationships unspecified, allowing flexibility in the learned geometry.

\textbf{Self-Supervised Representations in RL} have been explored both as auxiliary objectives in model-free reinforcement learning~\citep{schwarzer2021dataefficientreinforcementlearningselfpredictive}, and as part of latent model-based approaches~\citep{Schrittwieser_2020, hansen2024tdmpc2scalablerobustworld}.
A further line of work has focused on representation learning as a central component of practical methods, for instance enabling generalization in behavior cloning~\citep{lawson2026selfpredictiverepresentationscombinatorialgeneralization}, or robustness in value-based learning~\citep{zhang2021learninginvariantrepresentationsreinforcement}.
Self-supervised representations are also a core mechanism in unsupervised~\citep{laskin2021urlbunsupervisedreinforcementlearning} and zero-shot reinforcement learning~\citep{touati2023doeszeroshotreinforcementlearning}.
In particular, contrastive RL~\citep{eysenbach2023contrastivelearninggoalconditionedreinforcement}, forward-backward representation~\citep{touati2021learningrepresentationoptimizerewards, zheng2026reallylearnrepresentationoptimize}, ICVF~\citep{ghosh2023reinforcementlearningpassivedata} and TD-JEPA~\citep{bagatella2025tdjepalatentpredictiverepresentationszeroshot} recover representations which can be connected to the successor measure; we highlight a similar connection for spectral alignment with a temporal kernel.
Finally, this class of kernels shares the motivation of early works on temporal distance regression~\citep{tian2020modelbasedvisualplanningselfsupervised}, which however do not extract representations explicitly.

\section{Additional Experiments}
\label{app:extra_exp}
\begin{appendixexperiment}
\subsection{Spectral recovery}
\label{app:spectral_recovery}
Theorem \ref{thm:recovery} characterizes optimal representations as row-normalized spectral embeddings of the teacher kernel, up to an orthogonal transformation. To examine this connection empirically, we measure how well linear probes recover individual kernel eigenvectors from the learned representations (Figure \ref{fig:spectral_conversion}). We observe that recovery is strongest for the leading eigenvectors and decreases for those associated with smaller eigenvalues. 

The lower correlations observed in the state kernel are largely the result of a less even distribution of eigenvalue size, concentrating most of the energy in the lower eigenvector indices. In terms of mean $R^2$ weighted by each eigenvector's normalized eigenvalue, state gets a 0.965, proprio a 0.987, and depth a 0.992.

\begin{figure}[H]
    \centering
    \includegraphics[width=1\linewidth]{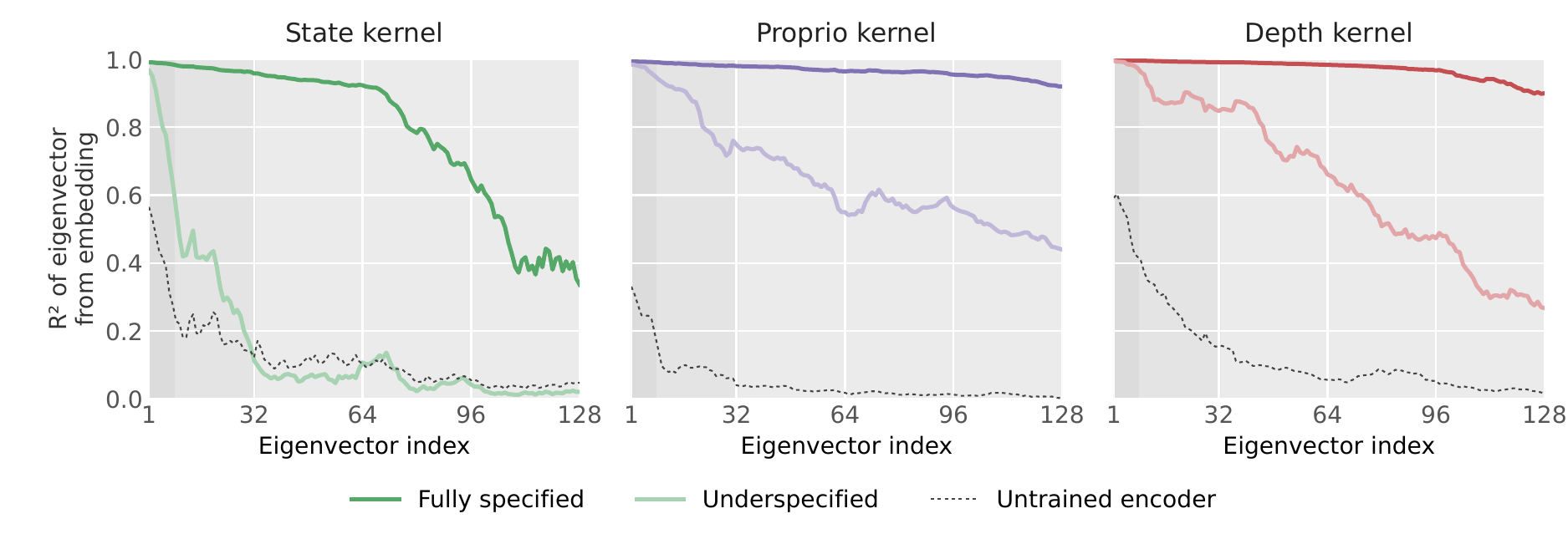}
    \caption{Linear probe $R^2$ between eigenvectors and embeddings from encoders trained on fully specified state, proprioception, and depth. Underspecified variants and an untrained ViT are included for comparison.}
    \label{fig:spectral_conversion}
\end{figure}
\end{appendixexperiment}
\begin{appendixexperiment}
\subsection{Prediction gradients and downstream performance}

In Figure~\ref{fig:stopgrad}, we additionally evaluate a variant of SpecWM which does not backpropagate the predictor regression objective through the target (i.e., the next state representations):
\begin{equation}
    \mathcal{L} = \mathcal{L}_\text{spec}
    - \lambda \sum_i \big(\cos(G,\text{sg}[H'])\big)_i
    \label{eq:total_loss_stopgrad}
\end{equation}
The prediction loss still updates the encoder through the current observation, but no longer through the next observation used as its target. We observe that planning success generally decreases, while probing results do not follow an obvious trend.

\begin{figure}[H]
    \centering
    \includegraphics[width=1\linewidth]{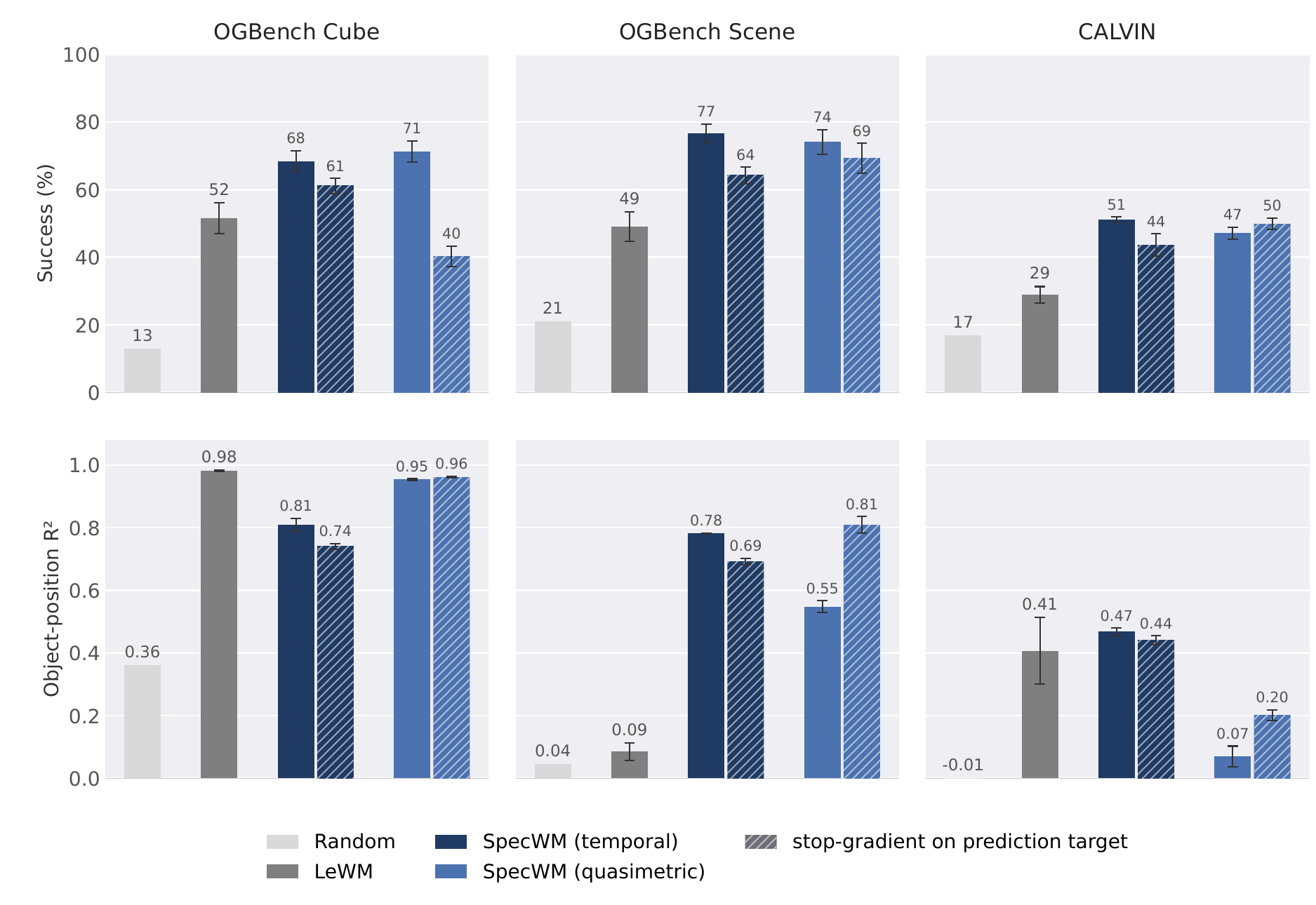}
    \caption{Effect of a stop-gradient on the prediction target. Bars report the mean over 3 seeds, with whiskers showing one cross-seed standard deviation. The stop-gradient generally lowers planning success and raises object recoverability in some cases while diminishing recoverability in others.}
    \label{fig:stopgrad}
\end{figure}
\end{appendixexperiment}

\begin{appendixexperiment}
\subsection{Data ablations}
\label{app:data_ablations}
It is known that predictors tend to generalize poorly to random actions when trained on expert-only data \citep{fu2021d4rldatasetsdeepdatadriven, kidambi2021morelmodelbasedoffline, yu2020mopomodelbasedofflinepolicy}. Here we ablate the failure data (Appendix~\ref{app:data}) by training LeWM and SpecWM with the temporal kernel on the play data alone, at matched updates, frames per update and schedule. As Figure~\ref{fig:data_ablation} shows, the failure data raises planning success for both models in every environment, but does not explain the difference between them.

\begin{figure}[H]
    \centering
    \includegraphics[width=1\linewidth]{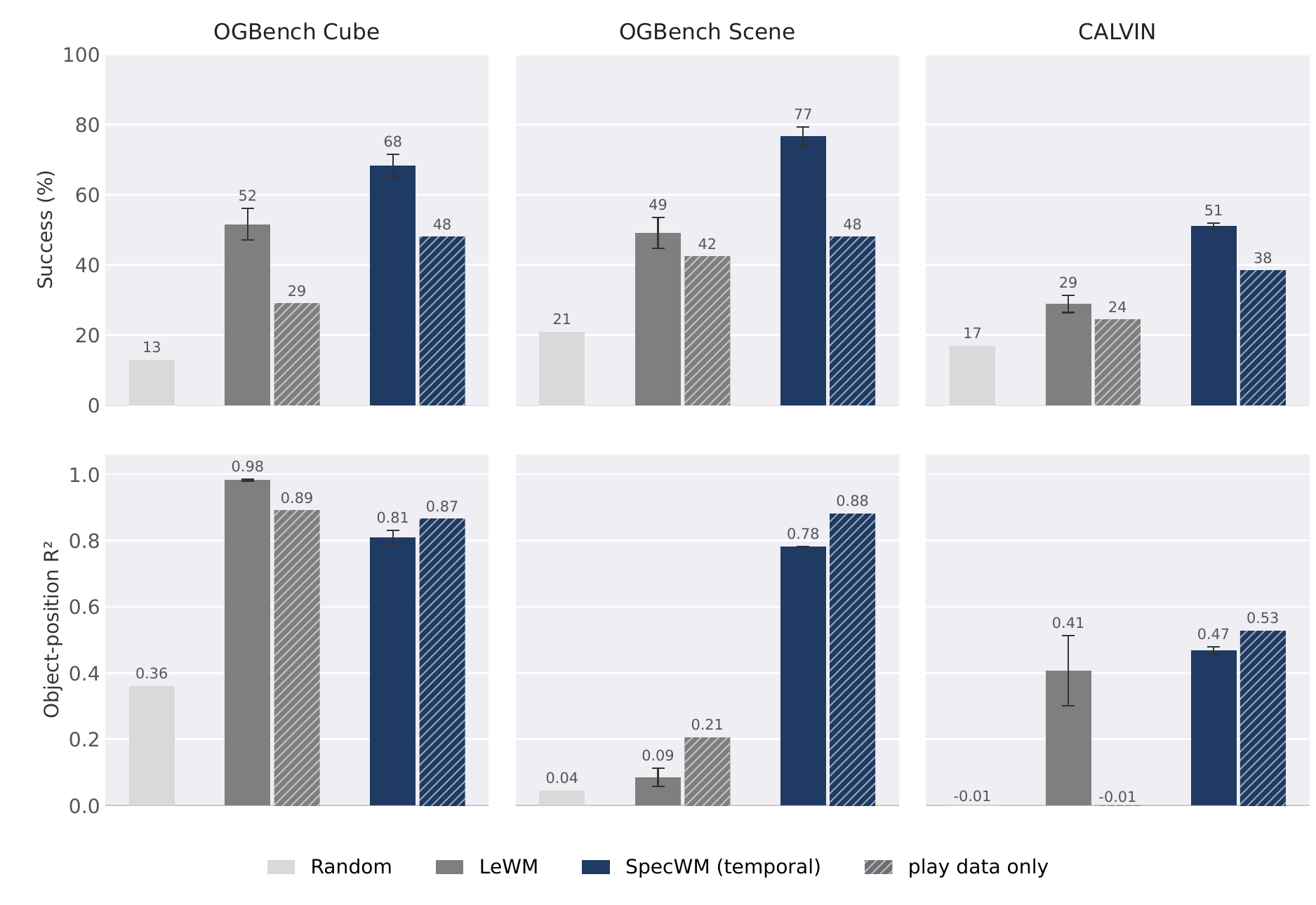}
    \caption{Effect of the failure data. CEM success (top) and object position probe $R^2$ (bottom) for LeWM and SpecWM (temporal), trained with the failure rollouts (solid, 3 seeds) or on the play data alone (hatched, 1 seed), at matched updates and frames per update. Failure data raises planning success for both models; SpecWM outperforms LeWM in both cases.}
    \label{fig:data_ablation}
\end{figure}
\end{appendixexperiment}

\begin{appendixexperiment}
\subsection{CEM cost ablation}
The two compared world models have different inherent distance metrics between samples in their latent distribution. For the purposes of the main paper, we use the one that fits each latent most naturally (L2 for LeWM and cosine similarity for SpecWM). Figure \ref{fig:cost-use-ablation} demonstrates that the choice of measure is not the cause of the planning gains.
\begin{figure}[H]
    \centering
    \includegraphics[width=0.45\linewidth]{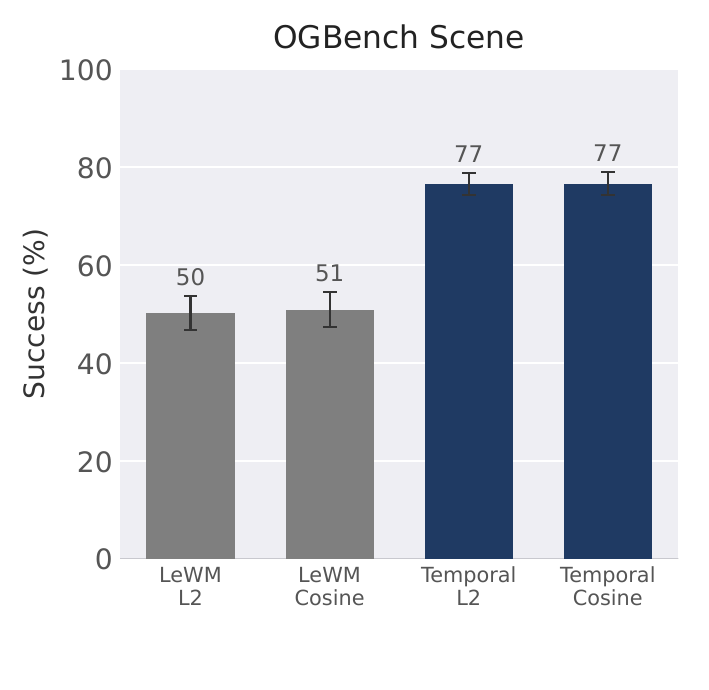}
    \caption{Effect of the choice of distance measure on CEM planning results. Cosine similarity and L2 distance between current embedding and target embedding are evaluated on both LeWM and SpecWM. The choice of distance measure appears to not be a major contributor to planning ability. Three training seeds are evaluated at their respective final three checkpoints for each setting.}
    \label{fig:cost-use-ablation}
\end{figure}
\end{appendixexperiment}

\begin{appendixexperiment}
\subsection{Normalization layer ablation}
To be faithful to the original LeWorldModel \citep{LeWorldModel} implementation, we use BatchNorm \citep{pmlr-v37-ioffe15} in the MLP producing the embedding from the ViT \citep{dosovitskiy2021an} encoder's CLS token. As SpecWM does not require BatchNorm to effectively optimize its objective, we use a LayerNorm for simplicity. As demonstrated in Figure \ref{fig:norm-ablation}, the choice of norm does not have a substantial impact on planning performance.
\begin{figure}[H]
    \centering
    \includegraphics[width=0.45\linewidth]{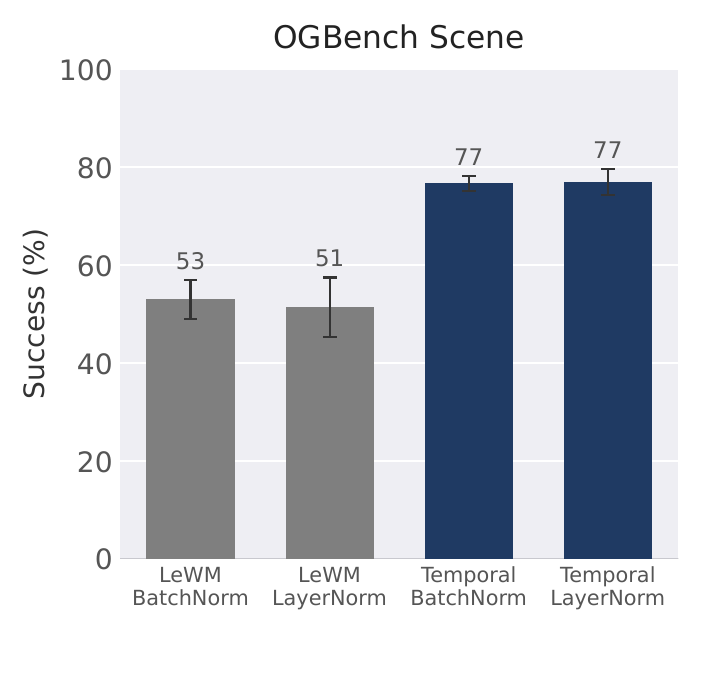}
    \caption{Effect of choice of normalization layer in the MLP converting the CLS token into an embedding on downstream CEM planning. One seed is used per run and CEM is reported as the average of the final three checkpoints of each run.}
    \label{fig:norm-ablation}
\end{figure}
\end{appendixexperiment}

\section{Reproducibility}
\label{reproducibility}
This appendix contains the necessary details to reproduce the results of the paper. For code, see: \url{https://anonymous.4open.science/r/SpecWM/}.

\subsection{Probing}
To measure to what extent the embeddings contain control-relevant objects, we perform linear probing of position coordinates. We sample 5,000 frames from the validation episodes, which the encoder never trained on, and split full episodes into training and testing sets of 80\% and 20\% respectively, repeated over five splits. Linear readouts use ridge regression, with all preprocessing and regularization
selection performed using training episodes only. Hyperparameters are selected using episode-grouped subsets of the training split. Ridge coefficients are selected from a logarithmic grid of 13 values between $10^{-2}$ and $10^{4}$. We then report held-out $R^2$ of the best ridge coefficient, using the full 80\% to fit the linear probe, averaged over the five splits.

\subsection{Long-range pairs on CALVIN}
\label{app:calvin_far}
In CALVIN, blocks move only when manipulated, so their positions often remain constant within a short clip. Within-clip temporal supervision therefore provides little signal to distinguish block positions. We augment each clip with 8 observations from the same episode, sampled $300$ to $3{,}000$ raw frames away (roughly 20 seconds to 3 minutes), to capture changes over longer periods. These observations contribute only to the spectral loss, using their original time indices, and require no associated actions.

To capture both short and long separations, we average two centered, trace-normalized Laplacian kernels with bandwidths $\sigma_1=7$ and $\sigma_2=45$, measured in sampled observations. We use the same construction for the quasimetric kernel:
\begin{equation}
    K_{ij} = \tfrac{1}{2}\Big(e^{-|i-j|/\sigma_1} + e^{-|i-j|/\sigma_2}\Big).
\end{equation}
Scene and Cube clips already capture the relevant changes, so we retain a single-scale kernel.

\subsection{Planning via CEM}
CEM hyperparameters are detailed in Table~\ref{tab:cem-config}. The base action sampling distribution is a Gaussian with starting mean and variance set according to expert data statistics. SpecWM cost-to-go is evaluated using cosine similarity and LeWM uses L2 distance.

\begin{table}[h]
\centering
\caption{CEM planning configuration.}
\begin{tabular}{lccc}
\toprule
\textbf{Parameter} & \textbf{Scene} & \textbf{Cube} & \textbf{CALVIN} \\
\midrule
Candidate sequences & 300 & 300 & 300 \\
CEM iterations & 30 & 30 & 30 \\
Elite sequences & 30 & 30 & 30 \\
Initial variance scale & 1.0 & 1.0 & 1.0 \\
Planning horizon (blocks) & 5 & 5 & 5 \\
Environment steps per block & 5 & 5 & 10 \\
Steps executed before replanning & 15 & 15 & 30 \\
Number of evaluation tasks & 100 & 100 & 100 \\
Goal offset (steps) & 25 & 25 & 50 \\
Evaluation budget (steps) & 60 & 60 & 100 \\
\bottomrule
\end{tabular}
\label{tab:cem-config}
\end{table}

\subsection{Training}

\paragraph{Architecture.} The JEPA architecture follows that of LeWorldModel and is consistent across objectives aside from normalization scheme. LeWorldModel uses BatchNorm in its projection head in line with \citet{LeWorldModel}, whereas SpecWM uses LayerNorm. Their remaining shared architectural settings are
summarized in Table~\ref{tab:model-architecture}.

\begin{table}[t]
\centering
\caption{Shared visual-encoder and action-conditioned dynamics architecture.}
\begin{tabular}{lp{0.66\linewidth}}
\toprule
Component & Configuration \\
\midrule
Visual encoder
& Randomly initialized ViT-T/14 with output dimension 192 \\
Projection head
& MLP $192\rightarrow2048\rightarrow192$ \\
Action encoder
& $1\times1$ Conv1D followed by an MLP with hidden dimension 768 and output dimension 192 \\
Dynamics predictor
& 6-layer action-conditioned transformer \\
Attention
& 16 heads with head dimension 64 \\
Transformer MLP dimension
& 2048 \\
Dropout
& 0.1 \\
Observation-history length
& 3 sampled observations \\
\bottomrule
\end{tabular}
\label{tab:model-architecture}
\end{table}
\begin{table}[t]
\caption{Environment-dependent action representation. Each action block contains all low-level actions between two consecutive sampled observations.}
\centering
\begin{tabular}{lccc}
\toprule
Environment & Frame skip & Raw action dim. & Action-block dim. \\
\midrule
OGBench Scene & 5  & 5 & 25 \\
OGBench Cube  & 5  & 5 & 25 \\
CALVIN        & 10 & 7 & 70 \\
\bottomrule
\end{tabular}
\label{tab:action-representation}
\end{table}

\paragraph{Optimization and training.}
Optimization hyperparameters are detailed in Table \ref{tab:training}. The unusual learning-rate values result from square-root batch-size scaling of base rates $5\times10^{-5}$ for AdamW and $10^{-3}$ for Muon. For CALVIN, SpecWM additionally encodes 8 distant observations per clip for the spectral loss, giving $24 \times (23+8)=744$ total images per update. LeWorldModel uses 552 images per update in all environments.

\begin{table}[t]
\centering
\caption{Optimization and student-training settings.}
\begin{tabular}{lp{0.68\linewidth}}
\toprule
Setting & Configuration \\
\midrule
Peak learning rates
& $1.461\times10^{-3}$ for matrices,
$2.922\times10^{-4}$ for convolutions, and
$7.30\times10^{-5}$ for biases and normalization parameters \\
MuonAdamW
& $(\beta_1,\beta_2)=(0.9,0.999)$, $\epsilon=10^{-8}$,
Muon momentum $0.95$, Nesterov updates, and RMS matching \\
Regularization
& Weight decay $10^{-3}$ on weights; no decay on biases or normalization
parameters; gradient clipping at $1.0$ \\
SpecWM
&33,375 updates; 552 base clip frames per update ($24 \times 23$);
$\lambda_{\mathrm{spec}}=1$ for Cube and Scene and $2$ for CALVIN;
cosine prediction-anchor weight $2$.\\
LeWorldModel
& 33,375 updates; 552 frames per update ($138\times4$);
prediction-MSE weight $1$ and SIGReg weight $0.09$ \\
\bottomrule
\end{tabular}
\label{tab:training}
\end{table}

\paragraph{Quasimetric training.}
We train the quasimetric teachers using a separate ViT-T/14 followed by a projection and an IQE head.
The IQE head's 384-dimensional output is divided into 24 IQE groups of
dimension 16, producing an asymmetric directed distance. The training objective expands distances between distinct observations while
constraining consecutive observations \citep{wang2023optimalgoalreachingreinforcementlearning}:
\[
\begin{aligned}
\mathcal{L}_{\mathrm{push}}
&=
-\mathbb{E}_{i\neq j}\!\left[d(z_i,z_j)\right],\\
\mathcal{L}_{\mathrm{step}}
&=
\mathbb{E}_t\!\left[
    \max\!\left(0,d(z_t,z_{t+1})-1\right)^2
\right],\\
\mathcal{L}_{\mathrm{QRL}}
&=
\mathcal{L}_{\mathrm{push}}
+\lambda\left(\mathcal{L}_{\mathrm{step}}-\epsilon^2\right).
\end{aligned}
\]
The dual variable is initialized to $1$ and updated with learning rate $0.05$
toward a constraint value of $\epsilon=0.25$, while being clipped to $[0,50]$.
Teachers are trained for 9,500 updates using 552 frames per update
($24$ clips of $23$ frames).

\subsection{Data}
\label{app:data}
For all environments, we supplement the official datasets with failure data to aid in predictor training using noised actions. For each
episode or segment $i$ of length $L_i$, we sample an onset
\[
t_i \sim \operatorname{Uniform}
\left(
\left\{\lfloor0.05L_i\rfloor,\ldots,\lfloor0.90L_i\rfloor\right\}
\right)
\]
and alternate the noise scale between $\sigma_i=0.3$ and $\sigma_i=0.6$.
The replayed action is
\[
\tilde a_t =
\begin{cases}
a_t, & t<t_i,\\
\operatorname{clip}(a_t+\sigma_i\epsilon_t,-1,1),
& t\geq t_i,
\end{cases}
\qquad
\epsilon_t\sim\mathcal{N}(0,I).
\]
Each replay is initialized from its recorded simulator state and follows the
recorded open-loop actions exactly before the noise onset. For Scene and Cube, noise is applied to all five action dimensions. For
CALVIN, noise is applied only to the six continuous motion dimensions, while
the binary gripper command is retained from the recorded trajectory.
The resulting Scene and Cube datasets each contain 1,000,000 frames from the
official dataset and 1,000,000 frames from our failure rollouts. The CALVIN
dataset contains 1,795,045 official frames and 1,793,322 failure frames.

\section{Full Proofs}
\label{proofs}

The proofs below follow the order of the stated theorems in Section \ref{theory}.

\subsection{Proof of Theorem~\ref{thm:recovery}}
\label{proof:recovery}

\lemmasfor{thm:recovery}
Let $H \in \R^{n \times d}$ be a free embedding table with nonzero rows $h_1, \dots, h_n$, and let $K^\alpha$ be the centered, tempered target. The loss rescales each row to the length the target prescribes, forms the Gram matrix, and compares it
to the target after double centering:
\begin{equation}
    \label{eq:loss}
    \bar h_i = \sqrt{K^\alpha_{ii}}\,\frac{h_i}{\norm{h_i}},
    \qquad
    S = \bar H \bar H^\top,
    \qquad
    \mathcal{L}_\text{spec}(H) = \norm{CSC - K^\alpha}_F^2,
\end{equation}
with $C = I - \tfrac1n\mathbf{1}\mathbf{1}^\top$.
Note that the rescaling pins the diagonal of $S$ to that of the target,
\begin{equation}
    \label{eq:diag}
    S_{ii} = \norm{\bar h_i}^2 = K^\alpha_{ii},
\end{equation}
for every $H$.

Let $\Tilde K = CKC$ have eigenvalues
$\lambda_1 \geq \dots \geq \lambda_n$ with unit eigenvectors $u_1, \dots, u_n$,
and let $m \geq 1$ be the number of directions the tempering keeps (the minimum of the count of positive $\lambda_\ell$ and the truncation count). Since the tempered spectrum vanishes outside these $m$
directions, the target becomes
\begin{equation}
    \label{eq:target}
    K^\alpha = \sum_{\ell=1}^{m} s_\ell^2\, u_\ell u_\ell^\top,
    \qquad
    s_\ell = \left(\frac{n\lambda_\ell^\alpha}
                        {\sum_{j \leq m}\lambda_j^\alpha}\right)^{1/2}.
\end{equation}
Three properties of the target are used below:
\begin{equation}
    \label{eq:facts}
    K^\alpha  \succeq 0,
    \qquad
    \operatorname{rank} K^\alpha  = m,
    \qquad
    K^\alpha \mathbf{1} = 0 .
\end{equation}
The first two come from \eqref{eq:target}, whose $m$ terms have
orthogonal directions $u_\ell$ and positive weights $s_\ell^2$. For the
third, $C\mathbf{1} = 0$ gives $\Tilde K\mathbf{1} = 0$, and every $u_\ell$ in
\eqref{eq:target} has $\lambda_\ell > 0$, so each is orthogonal to $\mathbf{1}$.
As assumed in Theorem~\ref{thm:recovery}, the target has
strictly positive diagonal entries.

\begin{prooflemma}
\label{lem:recovery-centered-gram}
$\mathcal{L}_\text{spec}(H) = 0$ if and only if $S = K^\alpha $.
\end{prooflemma}

\begin{proof}
If $S = K^\alpha $ the loss is $\norm{CK^\alpha C - K^\alpha }_F^2$. Since \(K^\alpha \) is symmetric and $K^\alpha\mathbf1=0$, we have $\mathbf1^\top K^\alpha =0$. Thus, $CK^\alpha  = K^\alpha  - \tfrac1n\mathbf{1}(K^\alpha \mathbf{1})^\top = K^\alpha $ and likewise $K^\alpha C = K^\alpha  - \tfrac1n (K^\alpha \mathbf{1}) \mathbf{1}^\top =K^\alpha $. Hence $CK^\alpha C = K^\alpha $ and the loss is $0$.

Conversely, suppose $CSC = K^\alpha $. Centering is not injective, so this alone does not give $S = K^\alpha $. Accordingly, we use the diagonal constraint \eqref{eq:diag}. Since $C$ is symmetric and idempotent,
\begin{align*}
    \operatorname{tr} K^\alpha 
      &= \operatorname{tr}(CSC) \\
      &= \operatorname{tr}(SC^2) \\ 
      &= \operatorname{tr}(SC) \\ 
      &= \operatorname{tr} S - \tfrac1n\,\mathbf{1}^\top S\mathbf{1}.
\end{align*}
Since $\operatorname{tr} S = \operatorname{tr} K^\alpha $ by \eqref{eq:diag}, we have that $\mathbf{1}^\top S \mathbf{1}=0$. Expanding, we see that $\mathbf{1}^\top S\mathbf{1} = \mathbf{1}^\top\bar H \bar H^\top \mathbf{1}= \norm{\bar H^\top\mathbf{1}}^2 = 0$, which forces $\bar{H}^\top \mathbf{1}=0$. Thus, we have that $\bar H\bar H^\top\mathbf{1} = S\mathbf{1} = 0$. Since $S$ is symmetric, $S\mathbf1=0$ also gives $\mathbf1^\top S=0$. Hence $CS=SC=S$, and therefore $S=CSC=K^\alpha$.
\end{proof}

\begin{prooflemma}
\label{lem:recovery-gram-factorization}
Assume \(m\le d\). Let \(U_m=[u_1,\ldots,u_m]\) contain the retained orthonormal eigenvectors of the centered kernel \(\widetilde K=CKC\).
Define $H_{\mathrm{spec}}\in\mathbb R^{n\times d}$ by assigning
sample $i$ the row
\[
    (H_{\mathrm{spec}})_{i,:}
    =
    \bigl(s_1u_1(i),\ldots,s_mu_m(i),
    \underbrace{0,\ldots,0}_{d-m}\bigr),
\]
where $u_\ell(i)$ is the $i$-th entry of eigenvector $u_\ell$
and the scales $s_\ell$ are given by equation~\ref{eq:target}. Then $S=K^\alpha $ if and only if
$H=\operatorname{diag}(c)H_{\mathrm{spec}}Q$
for some $c\in\mathbb R_{>0}^n$ and orthogonal $Q\in\mathbb R^{d\times d}$.
\end{prooflemma}

\begin{proof}

First, we will establish that $H_\text{spec}H_\text{spec}^\top=K^\alpha $. By equation \ref{spec_def} we have that
$H_\text{spec} = U_d \operatorname{diag}(s_1, \dots, s_d)$. Thus,
\begin{align*}
    H_\text{spec}H_\text{spec}^\top &= U_m \operatorname{diag}(s_1^2, \dots, s_m^2)\, U_m^\top \\
      &= \sum_{\ell=1}^{m} s_\ell^2\, u_\ell u_\ell^\top \\
      &= K^\alpha ,
\end{align*}
where the last equality is \eqref{eq:target}. The diagonal of a Gram matrix lists the squared row norms, so $\norm{(H_\text{spec})_i}^2 = K^\alpha _{ii}$. Second, we use a standard fact about Gram matrices: if $BB^\top = B'B'^\top$ for
$B, B' \in \R^{n \times d}$, then $B = B'Q$ for some orthogonal
$Q \in \R^{d \times d}$.

Suppose $S = K^\alpha $. Then $\bar H\bar H^\top = H_\text{spec}H_\text{spec}^\top$,
so $\bar H = H_\text{spec}Q$ for some orthogonal $Q$. By \eqref{eq:loss},
$h_i = c_i \bar h_i$ with $c_i = \norm{h_i}/\sqrt{K^\alpha_{ii}} > 0$, so
$H = \operatorname{diag}(c)\,\bar H = \operatorname{diag}(c)\,H_\text{spec}Q$.

Suppose instead $H = \operatorname{diag}(c)\,H_\text{spec}Q$. Since $Q$
preserves lengths, row $i$ of $H_\text{spec}Q$ has the same length as row $i$
of $H_\text{spec}$, namely $\sqrt{K^\alpha _{ii}} > 0$. The rows of $H$ are therefore
nonzero, and rescaling them to length $\sqrt{K^\alpha _{ii}}$ removes the factor
$c_i$, so $\bar H = H_\text{spec}Q$. Hence
$S = \bar H\bar H^\top = H_\text{spec}QQ^\top H_\text{spec}^\top
   = H_\text{spec}H_\text{spec}^\top = K^\alpha $.
\end{proof}

\begin{restatedresult}{Theorem}{thm:recovery}{Spectral recovery}
Suppose $m \leq d$. Then $\min_H \mathcal{L}_\text{spec}(H) = 0$, attained
exactly at
\begin{equation}
    H^\star = \operatorname{diag}(c)\, H_\text{spec}\, Q,
\end{equation}
with $c \in \R^n_{>0}$, $Q \in \R^{d \times d} \text{ orthogonal}$, and $H_\text{spec}$ the truncated spectral embedding of \eqref{spec_def}.
\end{restatedresult}

\begin{proof}[Proof of Theorem~\ref{thm:recovery}]
By Lemma~\ref{lem:recovery-centered-gram}, $\mathcal{L}_\text{spec}(H) = 0$ if and only if
$S = K^\alpha$, and by Lemma ~\ref{lem:recovery-gram-factorization} this holds if and only if $H$ has the form
\eqref{eq:minimizers}. The set of such $H$ is nonempty, since it contains
$H_\text{spec}$ itself ($c = \mathbf{1}$, $Q = I$), so the minimum is attained
and equals $0$.
\end{proof}

\subsection{Proof of Theorem \ref{thm:linear-recovery}}
We next connect this embedding geometry to linear probing performance. In particular, a kernel constructed from label relationships can make the labels themselves linearly recoverable from the embeddings, without directly regressing encoder outputs onto them. We formalize this below for row-normalized and centered vector-valued labels $Y$.

\begin{restatedresult}{Theorem}{thm:linear-recovery}{Linear recoverability}
Let $Y\in\mathbb R^{n\times p}$ satisfy $\mathbf1^\top Y=0$, row norms $\|y_i\|=1$, and $\operatorname{rank}(Y)\le d$. For $K=YY^\top$, $\alpha=1$, and all positive spectral directions retained, every global minimizer $H^\star$ of $\mathcal L_{\text{spec}}$ subject to $\|h_i\|=\sqrt d$ satisfies
\begin{equation}
    \label{eq:restate-zero-linear-probe}
    \min_{W\in\mathbb R^{d\times p}}\|Y-H^\star W\|_F^2=0.
\end{equation}
\end{restatedresult}

\begin{proof}[Proof of Theorem~\ref{thm:linear-recovery}]    We will construct $W$ that satisfies \ref{eq:restate-zero-linear-probe}. Since $Y$ is centered and its rows have unit norm, we have that $\Tilde{K}=CYY^TC=YY^T$ with $\operatorname{tr}(\Tilde{K})=n$. With $\alpha=1$ and all positive directions retained, the target is therefore $K^\alpha=YY^T$, with $K^\alpha_{ii}=1$. Write the compact SVD of $Y$ and $YY^T$ as  
    \begin{equation*}
        Y=U_m\Sigma_mV^T, \qquad YY^T=U_m\Sigma^2_mU^T_m. 
    \end{equation*}
    To pad the $m$-dimensional spectral embeddings to the encoder's $d$ dimensions, let $E_m=[I_m \;\; 0]\in \mathbb{R}^{m\times d}$. The spectral embedding is
    \begin{equation*}
        H_{\text{spec}}=U_m\Sigma_mE_m.
    \end{equation*}
    The rows of $H_{\mathrm{spec}}$ have unit norm because
$H_{\mathrm{spec}}H_{\mathrm{spec}}^\top=K^\alpha$ and
$K^\alpha_{ii}=1$. Multiplying by $\sqrt d$ makes each row
satisfy the required norm $\sqrt d$, while preserving zero
loss by Theorem~\ref{thm:recovery}. Since the loss is
nonnegative, every global minimizer under this row-norm
constraint must therefore also achieve zero loss. Further by Theorem \ref{thm:recovery}, each such minimizer has the form

    \begin{equation*}
        H^\star=\operatorname{diag}(c)\,H_{\text{spec}}Q.
    \end{equation*}
    Since $Q$ preserves lengths, the RMS constraint forces $c_i=\sqrt{d}$ for every $i$. Hence, we have that
    \begin{equation*}
        H^\star = \sqrt{d}U_m\Sigma_mE_mQ.
    \end{equation*}
    We can now define $W^\star$:
    \begin{equation}
        W^\star = \frac{1}{\sqrt{d}}Q^TE^T_mV^T.
    \end{equation}
    Using $QQ^T=I_d$ and $E_mE_m^T=I_m$, we have that
    \begin{align*}
        H^\star W^\star &= U_m\Sigma_mE_mQ Q^TE^T_mV^T\\
        &=U_m\Sigma_mV^T\\
        &=Y.
    \end{align*}
    Therefore $\min_W\|Y-H^\star W\|_F^2=0$.
\end{proof}

\subsection{Proof of Theorem~\ref{thm:kni-budget}}
\begin{restatedresult}{Theorem}{thm:kni-budget}
{Abstraction of task-irrelevant information}
Under the coverage assumption, fixed-norm encoder outputs, and the target-diagonal rescaling in Equation~\ref{eq:rescaling}, if each trajectory's target matrix $K^\alpha$ has
strictly positive diagonal entries almost surely and
$\mathbb E\,[\mathcal L_{\mathrm{spec}}]=0$, then for almost every clip and every frame $i$,
\begin{equation*}
    f_\theta((s_i,n))=f_\theta((s_i,n'))
\end{equation*}
for almost every pair $(n,n')$ drawn independently from
$\nu(\cdot\mid s_i)$.
\end{restatedresult}

\begin{proof}[Proof of Theorem~\ref{thm:kni-budget}]
We will show that coverage allows independent nuisance draws
while preserving the target inner products. Centering then
forces their rescaled embeddings to sum to zero, so changing
one nuisance draw cannot change its embedding.

Since the loss is nonnegative,
$\mathbb E\,\mathcal L_{\mathrm{spec}}=0$ implies zero loss
almost surely. Lemma~\ref{lem:recovery-centered-gram} therefore
gives $S=K^\alpha$ almost surely. Fix a task-relevant trajectory
and teacher metadata $Z$ for which this equality holds with
conditional probability one and the coverage assumption \ref{assump:coverage} holds. Let $q$ be the
number of frames. For this fixed $Z$, define
\begin{equation*}
    b_i(n)=\sqrt{\frac{K^\alpha_{ii}}{d}}\,
           f_\theta((s_i,n)).
\end{equation*}
These are the rescaled embeddings used in the spectral loss.
Since the encoder outputs have norm $\sqrt d$, we have
$\|b_i(n)\|^2=K^\alpha_{ii}$. Now draw auxiliary nuisance values $\widehat n_1,\ldots,
\widehat n_q$ independently, with
$\widehat n_i\sim\nu(\cdot\mid s_i)$.
For each $i\neq j$, coverage implies
\begin{equation*}
    b_i(\widehat n_i)^\top b_j(\widehat n_j)
    =K^\alpha_{ij}
    \qquad\text{almost surely}.
\end{equation*}
If the set of pairs violating this equality had
positive probability under the independent draws, assumption (A1) would
assign it positive conditional probability in the trajectory, causing non-zero loss. For diagonal entries, we also have
$b_i(\widehat n_i)^\top b_i(\widehat n_i)=K^\alpha_{ii}$
by the definition of $b_i.$
Together with the off-diagonal equalities above, this gives
\[
    b_i(\widehat n_i)^\top b_j(\widehat n_j)
    =K^\alpha_{ij}
    \qquad\text{for all }i,j
    \quad\text{almost surely}.
\]
Using this equality together with the fact that $K^\alpha$ is centered,
\begin{equation*}
    \left\|\sum_{j=1}^q b_j(\widehat n_j)\right\|^2
    =\sum_{j,k=1}^q
      b_j(\widehat n_j)^\top b_k(\widehat n_k)
    =\sum_{j,k=1}^q K^\alpha_{jk}
    =\mathbf1^\top K^\alpha\mathbf1
    =0
    \qquad\text{almost surely}.
\end{equation*}
Fix a frame $i$ and replace only $\widehat n_i$ by an
independent draw $\widehat n_i'\sim\nu(\cdot\mid s_i)$.
The new tuple has the same distribution, so both sums below
vanish almost surely:
\begin{equation*}
    b_i(\widehat n_i)+\sum_{j\neq i}b_j(\widehat n_j)=0,
    \qquad
    b_i(\widehat n_i')+\sum_{j\neq i}b_j(\widehat n_j)=0.
\end{equation*}
Subtracting gives
$b_i(\widehat n_i)=b_i(\widehat n_i')$ almost surely.
Since $K^\alpha_{ii}>0$, the rescaling can be cancelled,
yielding
\begin{equation*}
    f_\theta((s_i,\widehat n_i))
    =f_\theta((s_i,\widehat n_i'))
    \qquad\text{almost surely}.
\end{equation*}
The two nuisance draws are independent with distribution
$\nu(\cdot\mid s_i)$, which is the claimed invariance.
The argument applies to every frame and almost every $Z$.
\end{proof}

\subsection{Proof of Theorem~\ref{prop:pop-target}}

We resume from the assumptions stated in Section~\ref{sec:successor}, from the population-level variant of the spectral alignment objective
\begin{equation}
  \mathcal L_{\mathrm{spec}}^{\mathrm{pop}}(h) \;=\; \mathbb E_{\tau, i,j} \Big[
  \big(s(x_i,x_j) - \gamma^{|i-j|}\big)^2\Big],
\end{equation}
with $\tau:=(x_0, \dots, x_{T-1}) \sim \mathcal{D}$ and $(i,j)\sim\mathcal U(\{0,\dots,T-1\}^2)$, and from the expected teacher similarity of the state pair
\begin{equation}
    R_T(x, x') = \mathbb E_{\tau, i, j} [\gamma^{|i-j|}|x_i=x, x_j=x']. 
\end{equation}

We can then apply a bias-variance decomposition and rewrite the loss as:
\begin{align}
  \mathcal{L}_\mathrm{spec}^{\mathrm{pop}}(h) & = \mathbb E_{\tau, i, j}\Big[(s(x_i, x_j)-\gamma^{|i-j|})^2\Big]\\
  & = \mathbb{E}_{\tau, i, j}\Big[\big((s(x_i, x_j) - R_T(x_i, x_j)\big) + \big(R_T(x_i, x_j) - \gamma^{|i-j|})\big)^2\Big]\\
  & = \mathbb{E}_{x, x'} \bigg[ \mathbb E_{(\tau, i,j)\sim p(\cdot\mid x_i=x,\, x_j=x')} \Big[\big( s(x_i, x_j)- R_T(x_i, x_j)\big)^2\Big] \\
  & + \mathbb E_{(\tau, i,j)\sim p(\cdot\mid x_i=x,\, x_j=x')}\Big[\big(R_T(x_i, x_j) - \gamma^{|i-j|}\big)^2\Big] \\
  & + 2 \mathbb E_{(\tau, i,j)\sim p(\cdot\mid x_i=x,\, x_j=x')}\Big[\big(s(x_i, x_j)- R_T(x_i, x_j)\big)\big(R_T(x_i, x_j) - \gamma^{|i-j|}\big)\Big] \bigg],
\end{align}
where $\mathbb{E}_{x,x'}$ is over pairs of observations sampled from the dataset.
The second term is not dependent on $h$ and the third term fades as $(s(x_i,x_j)-R_T(x_i, x_j))$ is constant upon conditioning on $(x_i=x, x_j=x')$, and $\mathbb{E}_{(\tau, i,j)\sim p(\cdot\mid x_i=x,\, x_j=x')}[R_T(x_i, x_j) - \gamma^{|i - j|}] = 0$; thus minimizing $\mathcal{L}_{\mathrm{spec}}^{\mathrm{pop}}(h)$ is equivalent to minimizing $\mathbb{E}_{\tau, i, j} \big[(s(x_i, x_j)- R_T(x_i, x_j))^2\big]$.
We will now study the target $R_T$ in the limit of long trajectories, and connect it to the successor measure $M^\beta$ of the behavioral policy $\pi_\beta$.

\begin{restatedresult}{Theorem}{prop:pop-target}{Population target}
Under the setup above, with a finite state space, an irreducible, aperiodic Markov chain $P$ induced by $\pi_\beta$ with stationary distribution $\mu$ and $x_0\sim\mu$, for all $x,x'\in\mathcal S$, as $T\to\infty$,
\begin{equation*}
    T\,R_T(x,x') \;\longrightarrow\;
    \frac{M^\beta(x,x')}{\mu(x')} + \frac{M^\beta(x',x)}{\mu(x)}
    - \frac{\mathbf 1[x=x']}{\mu(x)} .
\end{equation*}
Moreover, if the induced Markov chain is reversible, i.e., $\mu(x)P(x'|x)=\mu(x')P(x|x')$,
\begin{equation*}
    T\,R_T(x,x') \;\longrightarrow\;
    \frac{2M^\beta(x,x')}{\mu(x')} - \frac{\mathbf 1[x=x']}{\mu(x')} .
\end{equation*}
\end{restatedresult}

\paragraph{Proof of Theorem~\ref{prop:pop-target}.}
By definition, $T R_T(x,x') = \frac{N(x,x')}{D(x,x')}$ with
\begin{equation}
    N(x,x'):= T \;\mathbb E_{\tau, i, j}\big[\gamma^{|i-j|}\,\mathbf 1[x_i=x,x_j=x']\big], \quad D(x,x') := E_{\tau, i, j}\big[\mathbf 1[x_i=x,x_j=x']\big].
\end{equation}
We will study the numerator and the denominator separately.

\emph{Numerator.} 
Let us consider the index difference $k:=j-i$, which is distributed as $p(k)=(T-|k|)/T^2$ for $|k|\le T-1$, and denote the probability that $x_i=x$ and $x_j=x'$ as
\begin{equation*}
    q_k(x,x') := \Pr(x_i = x, x_j = x' \mid j-i=k) =
    \begin{cases}
        \mu(x)P^{k}(x,x'), & k\geq0,\\
        \mu(x')P^{|k|}(x',x), & k<0.
    \end{cases}
\end{equation*}

We then have
\begin{align}
    N(x,x') & = T\; \mathbb E_{\tau, i, j}\big[\gamma^{|i-j|}\,\mathbf 1[x_i=x,x_j=x']\big] \\
    & = T\; \mathbb{E}_{i,j}\big[ \gamma^{|i-j|} \Pr(x_i=x, x_j=x'|i,j)\big] \\
    & = T\; \mathbb{E}_{i,j} \big[\gamma^{|i-j|} q_{j-i}(x, x')\big] \\
    & \stackrel{k=j-i}{=}  T\;\mathbb{E}_{k}\big[ \gamma^{|k|} q_{k}(x, x')\big] \\
    & =  T\; \sum_{k=-T+1}^{T-1}p(k) \gamma^{|k|} q_k(x,x') \\
    & =  T\;\sum_{k=-T+1}^{-1}p(k) \gamma^{|k|} q_k(x,x') + T\sum_{k=0}^{T-1}p(k) \gamma^{|k|} q_k(x,x') \\
    & = T\; \sum_{k=-T+1}^{-1}\frac{T+k}{T^2} \gamma^{|k|} \mu(x')P^{|k|}(x', x)  +T \sum_{k=0}^{T-1}\frac{T-k}{T^2} \gamma^{k} \mu(x)P^{k}(x, x') \\
    & = \sum_{k=-T+1}^{-1}\frac{T+k}{T} \gamma^{|k|} \mu(x')P^{|k|}(x', x)  + \sum_{k=0}^{T-1}\frac{T-k}{T} \gamma^{k} \mu(x)P^{k}(x, x') \\
    & \stackrel{k'=-k}{=} \sum_{k'=1}^{T-1}\frac{T-k'}{T} \gamma^{k'} \mu(x')P^{k'}(x', x)  + \sum_{k=0}^{T-1}\frac{T-k}{T} \gamma^{k} \mu(x)P^{k}(x, x')\\
    & = \mu(x')f_T(x',x) + \mu(x)f_T(x,x') - \mu(x)\mathbf{1}[x=x'],
\end{align}

where $f_T(x,x')=\sum_{k=0}^{T-1}\frac{T-k}{T}\gamma^k P^k(x,x')$. Each summand is bounded by $\gamma^k$ and
$\sum_{k=0}^\infty \gamma^k < \infty$, so by dominated convergence
\begin{equation}
    f_T(x,x') \to \sum_{k=0}^\infty \gamma^kP^k(x,x')=M^\beta(x,x').
\end{equation}
Thus, we conclude that
\begin{equation}
    N(x,x') \to \mu(x')M^\beta(x',x) + \mu(x)M^\beta(x,x') - \mu(x)\mathbf{1}[x=x'].
\end{equation}

\emph{Denominator.} We can deal with the denominator in a similar fashion to the numerator.
\begin{align}
    D(x,x') & = \mathbb E_{\tau, i, j}\big[\mathbf 1[x_i=x,x_j=x']\big] \\
    & = \mathbb{E}_{i,j} q_{j-i}(x, x') \\
    & \stackrel{k=j-i}{=}  \;\mathbb{E}_{k} q_{k}(x, x') \\
    & =  \sum_{k=-T+1}^{-1}p(k) q_k(x,x') + \sum_{k=0}^{T-1}p(k) q_k(x,x') \\
    & = \sum_{k=-T+1}^{-1}\frac{T+k}{T^2} \mu(x')P^{|k|}(x', x)  + \sum_{k=0}^{T-1}\frac{T-k}{T^2} \mu(x)P^{k}(x, x') \\
    & \stackrel{k'=-k}{=} \sum_{k'=1}^{T-1}\frac{T-k'}{T^2} \mu(x')P^{k'}(x', x)  + \sum_{k=0}^{T-1}\frac{T-k}{T^2}  \mu(x)P^{k}(x, x')\\
    & = \mu(x')g_T(x',x) + \mu(x)g_T(x,x') - \frac{1}{T}\mu(x)\mathbf{1}[x=x'],
\end{align}
with $g_T(x,x'):=\sum_{k=0}^{T-1}\frac{T-k}{T^2}P^k(x,x')$.
We have that $0 \leq \frac{T-k}{T^2} \leq \frac{1}{T}$ summing to $\frac{T+1}{2T}\to\frac{1}{2}$; moreover, due to irreducibility and aperiodicity, $P^k(x,x')\to\mu(x')$. Hence, by the Toeplitz lemma, $g_T(x,x') \to \frac{\mu(x')}{2}$, and finally $D(x,x')\to \mu(x)\mu(x')$.

Dividing the limits of $N(x,x')$ and $D(x,x')$ concludes the proof.

This concludes our analysis: embedding similarity is encouraged to regress a target that contains a symmetrized likelihood ratio, holding the behavioral successor measure as its numerator.
\subsection{Additional result: kernel dominance}

We have established that a label kernel can make a particular target linearly recoverable from embeddings. We now ask when one kernel produces a representation that is at least as useful as another for every linear probe task. This provides a notion of superiority of kernels.

Recall that $K^\alpha $ is the centered, tempered target. At the optimum, the \textit{rescaled} embeddings satisfy $\bar H \bar H^\top=K^\alpha $. However, a downstream probe reads the encoder output $H^\star$, whose rows have fixed norm $\sqrt{d}$. We therefore first express the geometry of $H^\star$ directly in terms of $K^\alpha $. Since $K^\alpha $ is the Gram matrix of $\bar H$, its diagonal entry $K^\alpha _{ii}$ is the squared length of row $\bar h_i$. Dividing each inner product by the lengths of the two rows gives their cosine similarity. Define
\begin{equation*}
    K^\text{cos}_{ik}= \frac{K^\alpha _{ik}}{\sqrt{K^\alpha _{ii}}\sqrt{K^\alpha _{kk}}}.
\end{equation*}
Thus $K^\text{cos}$ records the pairwise similarities after removing the row lengths prescribed by $K^\alpha$. In particular, $K^\text{cos}_{ii}=1$.
Recall also that the loss rescales each encoder output according to $\bar h_i = \sqrt{K^\alpha _{ii}} \frac{h^\star_i}{\sqrt{d}}$. At a global minimizer, it follows that
\begin{equation*}
    K^\alpha_{ik}=\langle \bar h_i, \bar h_k \rangle = \frac{\sqrt{K^\alpha _{ii}}\sqrt{K^\alpha _{kk}}}{d} \langle h^\star_i, h^\star_k \rangle.
\end{equation*}
Consequently, $\frac{1}{d}H^\star(H^\star)^\top = K^\text{cos}$. This identifies $K^\text{cos}$ as the cosine Gram matrix of the encoder output.

A scalar linear probe predicts $y$ as $H^\star w$, a linear combination of the columns of $H^\star$. Thus the exactly recoverable targets form $\operatorname{col}(H^\star)$. Using the identity $\operatorname{col}(HH^\top)=\operatorname{col}(H)$ and the relation $H^\star(H^\star)^\top=dK^\text{cos}$, we obtain
\begin{equation*}
    \operatorname{col}(H^\star)=\operatorname{col}(K^\text{cos}).
\end{equation*}

For $j\in \{1,2\}$, let $K^\alpha _j$ be the centered, tempered target constructed from $K_j$, with positive diagonal entries and rank at most $d$. Let $H^\star_j$ be any corresponding global minimizer subject to $\norm{h_i}=\sqrt{d}$. We say that $K_1$ \textbf{dominates} $K_2$ for linear probing if
\begin{equation}
    \min_W \norm{Y-H^\star_1 W}^2_F \le \min_W \norm{Y-H^\star_2 W}^2_F
\end{equation}
for every vector-valued target $Y$. Dominance is strict if the inequality is strict for at least one target.

\begin{theorem}[Kernel dominance]
\label{thm:kernel-dominance}    Under the assumptions above, $K_1$ dominates $K_2$ if and only if
    \begin{equation}
        \label{eq:kernel-dominance}
        \operatorname{col}(K^\text{cos}_2) \subseteq \operatorname{col}(K^\text{cos}_1).
    \end{equation}
    Dominance is strict if and only if this inclusion is strict.
\end{theorem}
\begin{proof}
    Suppose first that $\operatorname{col}(K^\text{cos}_2) \subseteq \operatorname{col}(K^\text{cos}_1).$ Since $\operatorname{col}(H^\star_j)=\operatorname{col}(K^\text{cos}_j)$, every column of $H^\star_2$ is a linear combination of the columns of $H^\star_1$. Hence there exists $T\in \mathbb{R}^{d\times d}$ such that $H^\star_2 = H^\star_1 T$. In particular, any prediction made from $H^\star_2$ can also be made from $H^\star_1$. To see this explicitly, let $W^\star_2$ minimize the probe error for a target $Y$. We have that
    \begin{align*}
        \min_W \norm{Y-H^\star_1W}^2_F &\le \norm{Y-H^\star_1TW^\star_2}^2_F\\
        &= \norm{Y - H^\star_2 W^\star_2}^2_F\\
        &= \min_W \norm{Y - H^\star_2 W}^2_F.
    \end{align*}
    Thus $K_1$ dominates $K_2$.

    Conversely, suppose $K_1$ dominates $K_2$. Take any scalar target $y \in \operatorname{col}(K^\text{cos}_2)$. This target is exactly recoverable from $H^\star_2$, so dominance implies
    \begin{equation*}
        0 \le \min_w\norm{y-H^\star_1w}^2 \le \min_w \norm{y-H^\star_2w}^2=0.
    \end{equation*}
    It is therefore also exactly recoverable from $H^\star_1$, giving $y \in \operatorname{col}(H^\star_1) = \operatorname{col}(K^\text{cos}_1)$. Since this holds for every $y \in \operatorname{col}(K^\text{cos}_2),$ the required inclusion follows. If the inclusion is strict, choose a target $y \in \operatorname{col}(K^\text{cos}_1) \setminus \operatorname{col}(K^\text{cos}_2)$. Its optimal probe error is zero under $K_1$ and positive under $K_2$, establishing strict dominance. If the column spaces coincide, their optimal probe errors agree for every target. Thus strict dominance requires strict inclusion.
\end{proof}

\end{document}